\documentclass[letterpaper,journal]{IEEEtran}
\usepackage{cite}
\usepackage{amsfonts,latexsym,amssymb}
\usepackage{amsmath,amsthm}
\usepackage{graphicx}
\usepackage[linesnumbered,ruled,vlined]{algorithm2e}
\SetAlgoNlRelativeSize{-2}
\SetAlCapSkip{0.5em}
\usepackage{graphics}
\usepackage{enumerate}
\usepackage{epsfig}
\usepackage{cases}
\usepackage{color}
\usepackage{url}
\usepackage{bm}
\usepackage{dsfont}
\usepackage{multirow}
\usepackage{booktabs}
\usepackage{flushend}
\usepackage{graphicx}
\usepackage{subcaption}
\usepackage{caption} 
\usepackage{mathrsfs}
\usepackage{mathtools}
\usepackage{makecell}
\usepackage{amsthm}
\usepackage[hidelinks]{hyperref}
\usepackage[colorinlistoftodos,prependcaption,textsize=tiny]{todonotes}

\theoremstyle{plain}
\newtheorem{theorem}{Theorem}
\newtheorem{lemma}{Lemma}

\newtheorem{proposition}{Proposition}

\theoremstyle{definition}

\theoremstyle{remark}

\def\BibTeX{{\rm B\kern-.05em{\sc i\kern-.025em b}\kern-.08em
    T\kern-.1667em\lower.7ex\hbox{E}\kern-.125emX}}

\IEEEoverridecommandlockouts
\begin{document}

\setlength{\abovedisplayskip}{6pt plus 2pt minus 2pt}
\setlength{\belowdisplayskip}{6pt plus 2pt minus 2pt}
\setlength{\abovedisplayshortskip}{3pt plus 1pt minus 1pt}
\setlength{\belowdisplayshortskip}{3pt plus 1pt minus 1pt}

\title{\huge 
End-to-End Latency-Minimizing and Load-Balanced Request Scheduling for Edge LLM Inference in Agentic AI Services
}

\author{Zhen Li,
        Jun Cai,~\textit{Senior Member, IEEE},
        Haoran Gao,
        An Li,
        and
        Tan Li
        
\IEEEcompsocitemizethanks{
\IEEEcompsocthanksitem Zhen Li, Jun Cai (corresponding author), An Li, and Haoran Gao are with the Department of Electrical and Computer Engineering, Concordia University, Montreal, QC, H3G 1M8, Canada. (E-mail: \{zhen.li, jun.cai\}@concordia.ca, \{haoran.gao, an.li\}@mail.concordia.ca).
\IEEEcompsocthanksitem Tan Li is with the Department of Computer Science, Hang Seng University of Hong Kong, Hong Kong SAR. (E-mail: tanli@hsu.edu.hk).
}}


\IEEEtitleabstractindextext{%
\begin{abstract}
Large language model (LLM)-powered agentic AI services increasingly demand low-latency inference, motivating the deployment of LLMs across distributed edge servers. However, heterogeneous communication and computing capabilities, together with dynamically evolving inference states, make the edge server selection for each incoming request time-varying and tightly coupled across slots. In this paper, we investigate an online request scheduling framework for edge LLM inference that jointly minimizes long-term average end-to-end latency and regulates workload distribution across heterogeneous edge servers. Two main challenges arise in this context. First, conventional latency models cannot accurately capture the fine-grained dynamics of multi-stage LLM execution. Second, the latency consequence of a scheduling decision is observed only after request completion, making immediate decision evaluation difficult. To address these challenges, we develop a cross-slot inference model that captures transmission, prefill, iteration-level decoding, and key-value (KV) cache evolution for each diverse request, and characterize server workload through a KV cache memory-time consumption metric. We propose the LYREO approach that transforms the long-term load-balancing constraint via Lyapunov optimization and employs reward redistribution with sequence-based return prediction to convert delayed outcomes into timely learning signals for earlier decisions. Simulations under various configurations demonstrate that LYREO consistently achieves lower latency and more balanced load distribution than representative learning-based and heuristic baseline schemes.

\end{abstract}

\begin{IEEEkeywords} Agentic AI, LLM inference, request scheduling, load balancing, reward redistribution.
\end{IEEEkeywords} }
\maketitle

\IEEEdisplaynontitleabstractindextext

\section{Introduction}
\IEEEPARstart{A}{gentic} artificial intelligence (agentic AI) is emerging as a key paradigm for intelligent networks, where autonomous applications continuously perceive environments, reason, and act with limited human intervention~\cite{jiang26jsac}. As the core intelligence component of agentic AI systems, large language models (LLMs) provide the reasoning and planning capabilities required for goal-directed autonomy.
Unlike conventional single-invocation applications, agentic AI repeatedly interacts with LLMs through observe-think-act cycles, so each invocation sits on the closed loop's critical path.
A delayed response directly postpones the next action, making inference latency a key determinant of system responsiveness~\cite{zhao26tccn2}.
Typically, LLM inference is hosted on remote cloud servers, 
but repeatedly forwarding inference requests to the cloud incurs substantial wide-area transmission delay, making agentic services vulnerable to network congestion and unstable connectivity. 
Edge deployment of LLMs
physically shortens transmission paths, mitigating cloud-side queuing and contention~\cite{zhang26comst}.
However, when multiple agentic AI applications share edge-hosted LLMs, their requests inevitably overlap in time and vary widely in input/output lengths and latency expectations, competing for edge servers' heterogeneous communication, computation, and memory resources.
Since the end-to-end latency is shaped by fluctuating wireless conditions, disparate server capacities and states, and request-specific characteristics, edge inference scheduling must be request-dependent, state-aware, and time-varying, rather than static and one-shot.

Although recent studies advanced edge LLM inference through efficient batching, task offloading, and resource allocation~\cite{zheng26tmc, he24tmc, huang25iotj}, critical issues closely tied to practical edge LLM inference remain insufficiently explored.
On one hand, unlike conventional edge-computing requests abstracted as atomic tasks completed within a single scheduling slot, LLM inference requests may remain active across multiple slots. During this period, subsequent requests continue to be scheduled and arrive, causing their service processes to overlap and interact through shared edge server resources.
For example, an edge server supporting multiple agentic-AI-driven mobile robots may receive a new reasoning request from one robot while still generating a response for another. Admitting the new request changes the allocation of shared inference resources, thereby affecting the completion time of ongoing inference.
Such interactions persist across scheduling slots, contrasting with conventional scheduling paradigms~\cite{li25arxiv, li25lcn, mek25icc}, where decision outcomes are immediately observed to guide the decision in the next time slot.
Therefore, LLM scheduling exhibits cross-request and cross-slot effects that most studies fail to capture: assigning a new request without accounting for unfinished ones prolongs latency, but this impact becomes observable only after the affected requests complete several slots later. 
This requires the scheduler to make sequential decisions before the outcomes of previous ones are revealed.

On the other hand, many inference scheduling studies prioritize latency while overlooking load imbalance among edge servers. 
A scheduler may continuously route requests to a high-capability server,
gradually pushing it toward resource limits. 
For LLM serving, such saturation is particularly risky because each admission decision creates a stateful resource commitment on the edge server. Since a request's resource demand continuously grows during response generation, redirecting new arrivals cannot immediately relieve an already saturated server.
This concentration leaves servers vulnerable to lengthy requests, traffic bursts, or scheduling errors, and such localized pressure could be amplified into system-wide tail-latency inflation and throughput loss, degrading the quality-of-service (QoS) of heavily loaded servers while leaving lightly loaded ones under-utilized~\cite{chen26arxiv, yi23tmc}.
Load balancing is thus not merely about resource utilization, but a vital long-term risk dispersion requirement missed by latency-only schedulers.

However, edge LLM inference scheduling introduces three key challenges.
\textit{First}, evaluating end-to-end latency requires modeling a request's complete lifecycle across wireless transmission and the inference pipeline, including batching, prefill, and decoding stages. Emerging LLM serving techniques such as continuous batching~\cite{yu22osdi,kwon23sosp} further allow requests to dynamically join or leave active batches, introducing complex dependencies that make realistic LLM inference substantially harder to model and analyze than conventional edge tasks.
\textit{Second}, conventional indicators like queue length cannot fully reflect actual workloads, as requests with diverse input/output token lengths consume varying inference resources over different duration.
It is necessary to map this resource consumption to a workload metric that also accounts for heterogeneous server capabilities~\cite{jin26infocom}.
Load balancing therefore relies on a metric capable of jointly capturing the intensity and duration of resource occupation across diverse requests and edge servers.
\textit{Third}, jointly addressing delayed latency feedback and long-term load balancing poses a further challenge.
The end-to-end latency impact of a scheduling decision remains unobservable until the inference request completes, leaving the scheduler without a timely learning signal~\cite{li25tsc}. Meanwhile, load balancing depends on workloads accumulated across slots and cannot be enforced by penalizing isolated decisions. This mismatch substantially complicates scheduler design.

To address these challenges, in this paper, we model each request's inference lifecycle and server-side resource occupation at a fine granularity, formulating an online scheduling problem that minimizes end-to-end latency subject to peak-memory feasibility and long-term load-balancing constraints.
This problem is difficult to solve due to its long-term constraints, request-server heterogeneity, and cross-slot delayed feedback.
We propose a \textbf{Ly}apunov-guided \textbf{r}eward-r\textbf{e}distribution \textbf{o}nline request scheduling (LYREO) approach that jointly addresses these difficulties through constraint-aware long-term guidance and delayed-feedback-aware policy learning. The main contributions are summarized as follows.

\begin{itemize}
    \item  We develop a fine-grained model that captures edge LLM inference throughout transmission, batching, prefill, and iteration-level decoding. 
    The model explicitly tracks the evolution of the key-value (KV) cache occupation on each edge server driven by unfinished requests and continuous batch updates,
    thereby quantifying how each assignment reshapes serving conditions and end-to-end latency.

    \item Building on this KV cache trajectory, we characterize each server's workload using a normalized KV cache memory-time consumption metric that jointly reflects memory intensity and active duration. Based on this metric, we formulate the long-term end-to-end latency-minimization problem subject to peak memory feasibility and inter-server load-balancing constraints.

    \item We put forth a novel approach, named LYREO, to address this optimization problem. LYREO leverages Lyapunov virtual queues to adaptively satisfy long-term load-balancing constraints. To effectively learn from delayed feedback, LYREO employs a long short-term memory (LSTM)-based reward redistribution mechanism, supported by sequence truncation and value bootstrapping, to provide timely learning signals for online policy learning.

    \item We evaluate the proposed LYREO approach across diverse system and algorithm configurations. The results show consistent improvements in end-to-end latency and load-balancing deviation over the baseline approaches.
\end{itemize}

The rest of this paper is organized as follows. Section~\ref{rw} reviews related work and highlights the novelty of this paper. 
Section~\ref{sys} presents the system model and problem formulation. 
Section~\ref{alg} presents LYREO and provides its theoretical analysis. 
Section~\ref{eval} reports and discusses the evaluation results, and Section~\ref{concl} concludes the paper.

\section{Related Work}\label{rw}
Recently, numerous studies have improved LLM inference serving through batching, scheduling, and memory management.
Orca~\cite{yu22osdi} introduced iteration-level scheduling, which executes LLM inference at the granularity of individual token-generation iterations rather than complete requests.
vLLM~\cite{kwon23sosp} introduced PagedAttention to dynamically manage the KV cache of requests with varying lengths.
Together, these techniques establish the foundation of continuous batching, substantially improving GPU utilization and latency for agentic AI applications~\cite{li24hpec}. Nonetheless, they primarily optimize inference engines or tightly connected accelerator clusters, leaving wireless transmission, time-varying user connectivity, and request assignment across geographically distributed edge servers outside the scheduling process.

Some recent studies have extended LLM inference to mobile edge networks. 
The wireless edge LLM inference framework in~\cite{zhang25twc} integrates request batching and model quantization, and jointly optimizes batch scheduling and resource allocation to maximize inference throughput under heterogeneous latency and accuracy requirements.
In~\cite{huang25iotj}, a two-timescale framework was introduced to coordinate slow-timescale LLM deployment with fast-timescale batch scheduling, GPU resource allocation, and bandwidth allocation in edge-cloud networks to minimize energy cost and end-to-end latency.
Additionally,~\cite{zhang25iotj} proposed a collaborative inference framework that jointly selects distributed devices and partitions the LLM into deployable shards to reduce inference latency and improve throughput.
Nevertheless, these studies generally treat each inference request as a conventional atomic task with simplified latency models, failing to capture the multi-stage and cross-slot dynamics of practical LLM serving.
Load balancing remains comparatively underexplored in edge LLM inference. 
A recent framework in~\cite{mou26tsc} combined workload prediction with reinforcement learning in scheduling to reduce inference latency and balance GPU utilization across edge LLM instances.
CoLLM~\cite{li24iccc} enabled collaborative LLM inference across resource-constrained devices to reduce inference latency while balancing energy consumption.
However, these load-aware methods typically overlook time-integrated resource occupation, and long-term load balancing of heterogeneous requests and edge servers remains largely unexplored.

\begin{figure*}[t]
	\centering\includegraphics[width=0.65\linewidth]{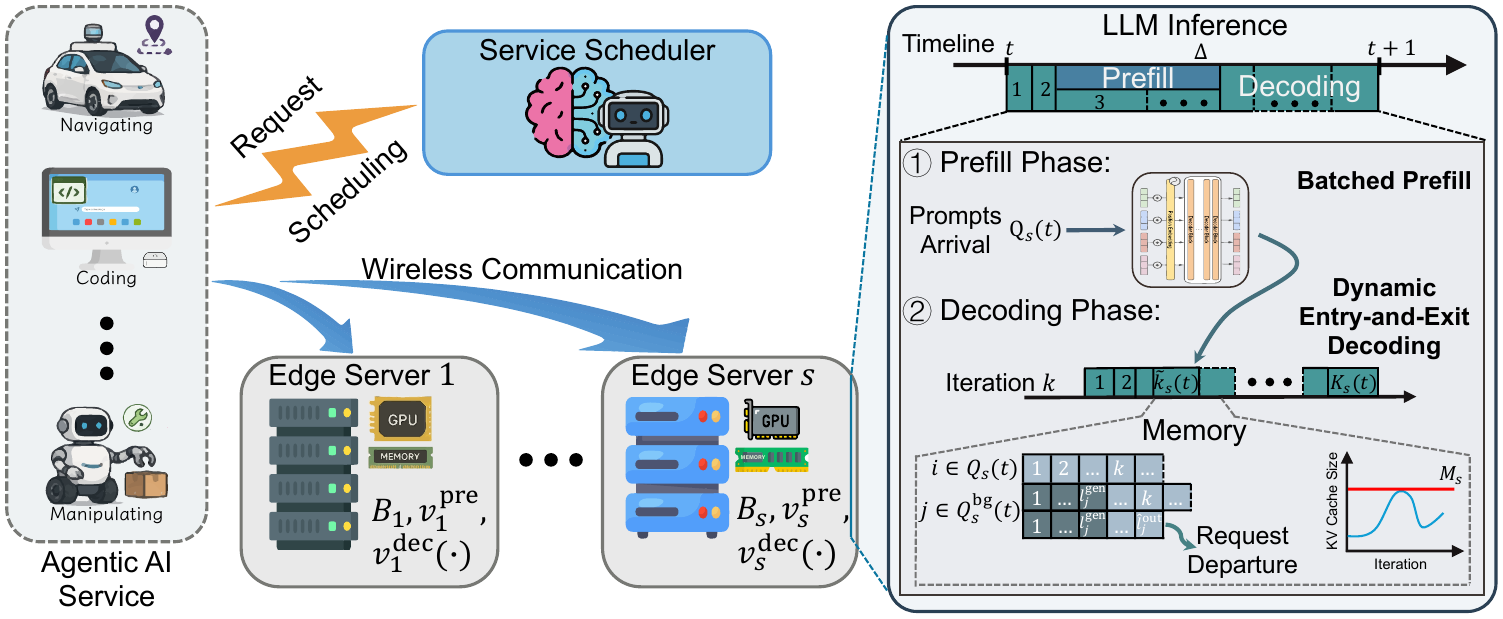}
	\caption{System model of the edge-assisted LLM inference framework.}
	\label{system}
	\vspace{-0.1in}
\end{figure*}

To improve performance in edge LLM inference, recent studies have widely applied deep reinforcement learning (DRL)~\cite{zhu26twc, younesi25ucc,qiu25fcn} and Lyapunov optimization~\cite{ma26infocom,tang26arxiv} to coordinate request scheduling, offloading, and resource allocation. 
Nonetheless, these studies rely on per-slot system observations and immediately available performance signals, without explicitly addressing the delayed nature of inference feedback. Consequently, how to attribute delayed outcomes to their causal scheduling decisions remains an open problem.

Unlike prior studies, this paper jointly models the cross-slot, multi-stage execution of edge LLM inference and characterizes the resulting edge server workload.
Furthermore, we develop a Lyapunov-guided reward-redistribution online scheduling approach that resolves the delayed feedback while maintaining long-term load balance across heterogeneous edge servers.
To our knowledge, these coupled issues have not been jointly addressed in previous work.


\section{System Model \& Problem Formulation}\label{sys}
In this section, we first present an overview of the considered edge-assisted LLM inference system in Section~\ref{S-A}. 
Sections~\ref{S-B}--\ref{S-D} then detail the communication, LLM inference, and memory models, respectively.
Finally, Section~\ref{S-E} formulates the corresponding optimization problem.

\subsection{System Overview}\label{S-A}
As illustrated in Fig.~\ref{system}, the edge LLM-based agentic AI system comprises an LLM service scheduler, multiple mobile users, and a set of edge servers denoted by $\mathcal{S} = \{1, 2, \ldots, S\}$. 
The system operates in equal discrete time slots $t \in \mathcal{T} = \{1, \ldots, T\}$ of length $\Delta$ seconds.
Let $\mathcal{U}=\{1,\ldots,U\}$ denote the set of subscribed mobile users with time-varying locations.
These mobile users represent devices running LLM-powered agentic applications, whose reasoning invocations are served by shared edge-hosted models. Operating at the inference-service layer, we represent these application-generated invocations using a slotted request model.
Specifically, at the beginning of each time slot $t$, every user $u\in\mathcal{U}$ generates one inference request, indexed by $i=(u,t)$. 
The inference requests generated by users in time slot $t$ are collected in set $\mathcal{Q}(t)=\{(u,t)\mid u\in\mathcal{U}\}$.
Each request $i\in \mathcal{Q}(t)$ is characterized by a tuple $\langle l_i^{\text{in}}, \hat{l}_i^{\text{out}}, \boldsymbol{o}_i \rangle$, where $l_i^{\text{in}} \in \mathbb{Z}^+$ and $\hat{l}_i^{\text{out}} \in \mathbb{Z}^+$ are the input and output token lengths, respectively.
Here, $\hat{l}_i^{\text{out}}$ is treated as a known input, whether user-specified or estimated via existing methods~\cite{zhu26twc}.
$\boldsymbol{o}_i \in \mathbb{R}^2$ denotes the geographical coordinates of the user when generating request $i$.\looseness=-1

The scheduling decision is denoted by $x_{i,s}(t) \in \{0,1\}$, where $x_{i,s}(t) = 1$ indicates that request $i\in\mathcal{Q}(t)$ is assigned to edge server $s$, and $x_{i,s}(t)=0$ otherwise.
The decisions are made by the service scheduler at the beginning of each time slot, based on the collected request metadata. 
To reflect practical LLM inference systems, we do not assume edge servers are idle at scheduling time, i.e., they may still be processing backlogged requests from previous time slots.
Let $\mathcal{Q}_s(t) = \{i \in \mathcal{Q}(t) \mid x_{i,s}(t) = 1\}$ be the subset of requests scheduled to edge server $s$ in time slot $t$. 
These requests are then transmitted to the edge servers and processed accordingly, as detailed in Sections~\ref{S-B} and~\ref{S-C}.
Upon completion, the inference results are returned to the users.

\subsection{Communication Model}\label{S-B}
Following the scheduling decisions, users transmit their input request tokens to the assigned edge servers via wireless links.
Let $B_s$ denote the total bandwidth of edge server $s$, which is equally shared among all requests scheduled to it. 
For request $i \in \mathcal{Q}_s(t)$, the transmission rate is given by
\begin{equation}
    R_{i,s}(t) = \frac{B_s}{|\mathcal{Q}_s(t)|} \log_2\left(1 + \frac{p_i h_{i,s}(t)}{N_{0,s}B_s/|\mathcal{Q}_s(t)|}\right),
\end{equation}
where $N_{0,s}$ is the noise power spectral density (PSD) at edge server $s$, and $p_i$ denotes the transmission power of the user generating request $i$.
$h_{i,s}(t)$ denotes the channel gain between the user and edge server $s$, modeled as a function of the distance $\|\boldsymbol{o}_i - \boldsymbol{o}_s\|$~\cite{li23tvt},
where $\boldsymbol{o}_s$ is the location of edge server $s$.
Then, the transmission delay of request $i$ is given by
\begin{equation}
    D^{\text{tx}}_{i,s}(t) = \frac{\beta l_i^{\text{in}}}{R_{i,s}(t)},
\end{equation}
where $\beta$ (in bits) is the data size of a single token.

Since inference computation begins only after all scheduled prompts have been received~\cite{li25lcn}, 
let $D_s^{\text{sy}}(t) = \max_{i\in \mathcal{Q}_s(t)} D^{\text{tx}}_{i,s}(t)$ be the synchronization delay for the batch at edge server $s$.
Correspondingly, the waiting delay for request $i$ is determined by the difference between this synchronization delay and its  transmission time, which is given by
\begin{equation}
    D^{\text{wa}}_{i,s}(t) = D_s^{\text{sy}}(t) - D^{\text{tx}}_{i,s}(t).
\end{equation}

Given that edge servers are equipped with significantly higher transmit power and wider downlink bandwidth compared to mobile users, the feedback latency after inference completion is considered negligible in this work.

\subsection{LLM Inference Model}\label{S-C}
Modern LLMs rely on decoder-only transformer architectures, comprising prefill and decoding phases. In contrast to previous studies that simplify this process via static batching and single-slot completion, this work incorporates a highly realistic serving framework driven by continuous batching and iteration-level dynamics. The two phases are detailed below.

\subsubsection{Prefill Phase}
In the prefill phase, the input tokens are processed to compute the first output token and generate the corresponding KV cache.
To fully exploit the parallel computing capability of GPUs, the tokens from multiple requests can be aggregated and processed simultaneously through large-scale matrix multiplications~\cite{kwon23sosp}.
Accordingly, all requests arriving at edge server $s$ within time slot $t$, i.e., $\mathcal{Q}_s(t)$, are grouped into a single batch for one forward-pass computation, known as batched prefill~\cite{zhang25twc}.
Since prefill is inherently compute-bound with relatively stable throughput, we assume its computation is independent of decoding tasks as long as sufficient KV cache memory is provisioned.
Let $v_s^{\text{pre}}$ (in tokens/s) be the prefill rate of edge server $s$.
The prefill latency of the batched requests scheduled to edge server $s$ in time slot $t$ is determined by the total number of input tokens and the edge server's computational capability, expressed as~\cite{huang25iotj, cheng24arxiv}
\begin{equation}
    D^{\text{pre}}_{s}(t) = \frac{\sum_{i \in \mathcal{Q}_s(t)} l_i^{\text{in}}}{v_s^{\text{pre}}}.
\end{equation}

\subsubsection{Decoding Phase}
In the decoding phase, the LLM generates output tokens autoregressively, producing one token per iteration until the output length is reached.
At each iteration, the inference engine takes the previously generated token as input and performs computation based on the KV cache of all preceding tokens, which is updated incrementally as decoding proceeds.
Unlike prefill, decoding iterations cannot be parallelized due to strict data dependencies. Consequently, despite minimal computation, decoding remains highly memory-bound because it requires repeatedly loading model weights and the accumulated KV cache.
We consider an advanced decoding mechanism enabled by iteration-level scheduling termed dynamic entry-and-exit decoding (DEED)~\cite{yu22osdi}, featuring: 1) \textit{early-finished eviction}, where completed requests are removed from the batch and their KV cache is released immediately from memory at the end of iteration; and 2) \textit{late-joining admission}, where new requests are admitted to the active batch and decoded together with backlogged requests,
provided sufficient KV cache capacity is available.

Based on the DEED mechanism, let $\mathcal{Q}^{\text{bg}}_s(t)$ be the set of backlogged requests remaining at the edge server $s$ at the beginning of time slot $t$.
Since the decoding phase typically dominates the end-to-end serving latency while the prefill latency is comparatively short~\cite{zhu26twc},
we assume, for a reasonable setting of $\Delta$, all backlogged requests in $\mathcal{Q}^{\text{bg}}_s(t)$ have already completed the prefill phase (including transmission) and remain only in the decoding phase.
Let $l_j^{\text{gen}}(t)$ be the cumulative number of generated tokens for request $j\in \mathcal{Q}^{\text{bg}}_s(t)$ at the beginning of slot $t$.
While newly scheduled requests undergo transmission and prefill, the inference engine continues decoding these backlogged requests in $\mathcal{Q}^{\text{bg}}_s(t)$.
Specifically, during the time budget $D_s^{\text{sy}}(t) + D^{\text{pre}}_{s}(t)$, the KV cache memory occupied on the GPU of edge server $s$ increases linearly with the generated tokens~\cite{li24hpec}. 
The KV cache memory occupation at the $k$-th decoding iteration in time slot $t$ on edge server $s$ can be derived as
\begin{equation}
    m_{s}^{\text{I}}(k,t) = \alpha \sum_{j \in \mathcal{G}_s(k,t)}(l^{\text{in}}_{j} + l^{\text{gen}}_{j}(t)+ k),
\end{equation}
where $\alpha$ is the per-token KV cache size\footnote{We approximate discrete KV cache allocation as continuous, as tail-block internal fragmentation is negligible relative to overall sequence lengths.}, and $\mathcal{G}_s(k,t) = \{j \in \mathcal{Q}^{\text{bg}}_s(t) \mid k \le \hat{l}_{j}^{\text{out}}-l^{\text{gen}}_j(t)\}$ denotes the set of active background requests at the $k$-th decoding iteration.
The condition $k \le \hat{l}_{j}^{\text{out}}-l^{\text{gen}}_j(t)$ in $\mathcal{G}_s(k,t)$ reflects the early-finished eviction of DEED, where the request is removed from the summation as its KV cache is released once decoding is complete.

Denote the decoding throughput of edge server $s$ (in steps/s) by $v^{\text{dec}}_s(\cdot)$, which is a function of the current total KV cache size. 
The time per decoding iteration on edge server $s$ when $m_{s}^{\text{I}}(k,t) > 0$ is given by
\begin{equation}\label{dectime1}
    \tau_s^{\text{I}}(k,t) = \frac{1}{(1-w\mathbb{I}_s(k,t)) \cdot v^{\text{dec}}_s(m_{s}^{\text{I}}(k,t))},
\end{equation}
where $\mathbb{I}_s(k,t)\in\{0,1\}$ indicates the operational status of inference engine.
Specifically, $\mathbb{I}_s(k,t) = 0$ when the server is in a pure decoding state, while $\mathbb{I}_s(k,t) = 1$ indicates the presence of concurrent prefill computation, which reduces the throughput by a fraction $w\in(0,1)$.
Building on the per-iteration decoding time, the number of decoding iterations that can be completed on edge server $s$ before the prefill phase finishes in time slot $t$ can be determined by
\begin{equation}
\begin{aligned}
    \tilde{k}_s(t)
    = \max \Bigl\{ k\Bigm|\;
    &\sum_{\kappa=1}^{k} \tau_s^{\text{I}}(\kappa,t)
    \leq D_s^{\text{sy}}(t) + D^{\text{pre}}_{s}(t), \\
    &\quad m_{s}^{\text{I}}(k,t) > 0
    \Bigr\},
\end{aligned}
\end{equation}
where the condition $m_{s}^{\text{I}}(k,t) > 0$ ensures that the iteration counts only when there are backlogged requests to be decoded.

Once the prefill phase finishes,
the newly prefilled requests enter the decoding batch immediately through the late-joining admission mechanism of DEED\footnote{We ignore the sub-iteration waiting time for new requests to align with iteration boundaries, as it is negligible compared to the slot duration and overall inference latency.}.
Let $\mathcal{H}_s(k,t) = \{i \in \mathcal{Q}_s(t) \mid k-\tilde{k}_s(t) \le \hat{l}_{i}^{\text{out}}\}$ denote the set of active new requests at the $k$-th decoding iteration.
Therefore, the total KV cache size on edge server $s$ can be calculated as 
\begin{equation}
\begin{aligned}
    m_{s}^{\text{II}}(k,t) = \ & \alpha \sum_{i \in \mathcal{H}_s(k,t)} (l^{\text{in}}_i + k  - \tilde{k}_s(t)) \\
    & + \alpha \sum_{j \in \mathcal{G}_s(k,t)} (l^{\text{in}}_{j} + l^{\text{gen}}_{j}(t) + k) ,
\end{aligned}
\end{equation}
where the first term accounts for the KV cache of active newly scheduled requests, and the second term accounts for that of active backlogged requests. 
Note that both the $m_{s}^{\text{I}}(k,t)$ and $m_{s}^{\text{II}}(k,t)$ are not necessarily monotonic with respect to $k$, as the early-finished eviction in DEED continuously removes completed requests and releases their KV cache, while the late-joining admission of new requests and the token generation of active requests increase the KV cache.
Given the total KV cache size, the per-iteration decoding time and the number of iterations completed between prefill and the end of time slot $t$ on edge server $s$ can be respectively expressed as 
\begin{equation}\label{dectime2}
    \tau_s^{\text{II}}(k,t) = \frac{1}{ v^{\text{dec}}_s(m_{s}^{\text{II}}(k,t))},
\end{equation}
\begin{equation}
\begin{aligned}
    \bar{k}_s(t)
    = \max \Bigl\{ k\Bigm|&\;
    \sum_{\kappa=1}^{k} \tau_s^{\text{II}}(\tilde{k}_s(t) + \kappa,t) \leq \Delta - D_s^{\text{sy}}(t) \\
    &\ - D^{\text{pre}}_{s}(t), \ m_{s}^{\text{II}}(\tilde{k}_s(t)+k,t) > 0
    \Bigr\}.
\end{aligned}
\end{equation}

Accordingly, for request $i$ scheduled to edge server $s$ in time slot $t$, the decoding latency is calculated by accumulating the per-iteration decoding latency over all output tokens, given by
\begin{equation}
    D_{i,s}^{\text{dec}}(t) = \sum_{k=1}^{\hat{l}^{\text{out}}_{i}} \tau'_{i,s}(k),
\end{equation}
where $\tau'_{i,s}(k)$ denotes the decoding time when generating the $k$-th output token of request $i$. 
Note that $\tau_s^{\text{I}}(k,t)$, $\tau_s^{\text{II}}(k,t)$, and $\tau'_{i,s}(k)$ represent the same physical quantity viewed from two different perspectives: $\tau'_{i,s}(k)$ tracks the per-iteration decoding time from the perspective of an individual request, 
whereas $\tau_s^{\text{I}}(k,t)$ and $\tau_s^{\text{II}}(k,t)$ characterize it along the system timeline in time slot $t$. 
Since the per-iteration decoding time along the system timeline has been fully derived in~\eqref{dectime1} and~\eqref{dectime2}, $\tau'_{i,s}(k)$ can be directly obtained via the corresponding cross-slot index mapping, which is omitted here for brevity.

For edge server $s$, the total number of completed decoding iterations within time slot $t$ is the sum of iterations from both phases, which is given by $K_s(t) = \tilde{k}_s(t) + \bar{k}_s(t)$.
Therefore, the set of backlogged requests at the beginning of time slot $t+1$ is updated as
\begin{equation}
\begin{aligned}
    \mathcal{Q}^{\text{bg}}_s(t+1)  
    =&\ \{i \in \mathcal{Q}_s(t) \mid \hat{l}_i^{\text{out}} > \bar{k}_s(t) \} \\
    &\cup \{j \in \mathcal{Q}^{\text{bg}}_s(t) \mid \hat{l}_j^{\text{out}} > l_j^{\text{gen}}(t) + K_s(t) \},
\end{aligned}
\end{equation}
where the first term represents the newly scheduled requests in $\mathcal{Q}_s(t)$ that have not been completed by the end of time slot $t$, and the second term represents the previously backlogged requests in $\mathcal{Q}^{\text{bg}}_s(t)$ that remain unfinished. 
In contrast to the conventional ``clear-then-schedule'' assumption, here we do not require all requests to be completed within the current time slot, which more faithfully reflects the realistic operation of LLM inference systems.


\subsection{Memory Analysis}\label{S-D}
Unlike traditional computing tasks, LLM serving, especially the dominant autoregressive decoding, is inherently stateful and frequently bottlenecked by memory-resource availability. Since the KV cache footprint strongly affects the achievable decoding throughput and serving capacity, KV cache occupation is explicitly modeled as a key resource state~\cite{sun24osdi, kwon23sosp}.
Based on the system workflow described above, the KV cache occupation at the $k$-th decoding iteration in time slot $t$ on edge server $s$ can be unified as 
\begin{equation}
    m_s(k,t) = 
    \begin{cases}
        m_{s}^{\text{I}}(k,t),\ 1\le k \le \tilde{k}_s(t),\\
        m_{s}^{\text{II}}(k,t),\ \tilde{k}_s(t) < k \le K_s(t) .
    \end{cases}
\end{equation}

Let $M_s$ be the GPU memory capacity of edge server $s$ reserved for KV cache during LLM inference.
To ensure feasibility, the peak KV cache occupation must not exceed the available memory capacity throughout the inference process, 
\begin{equation}
    M^{\text{pk}}_s(t) = \max_{{k=1,\dots,K_s(t)} } m_s(k,t) \le M_s.
\end{equation}
Note that the peak memory constraint defined above naturally unifies two feasibility requirements. 
First, it ensures that the KV cache memory does not overflow during the autoregressive decoding phase as tokens are incrementally generated. 
Second, it acts as an admission feasibility condition for newly scheduled requests, ensuring that sufficient GPU memory is available for the KV cache generated during prefill. As a result, admitted requests can start the prefill phase immediately after batching, without incurring additional engine-level waiting.

In addition to acting as a hard feasibility constraint, the KV cache occupation on edge servers reflects the continuous inference workload, as it remains persistently allocated throughout the inference lifecycle and continuously consumes GPU memory resources. 
Accordingly, we define the workload of edge server $s$ in time slot $t$ as the normalized KV cache memory-time consumption, i.e.,
\begin{equation}\label{load}
    \eta_s(t) = \frac{\sum_{k=1}^{K_s(t)}m_s(k,t)\tau_s(k,t)}{ M_s\Delta},
\end{equation}
where the unified per-iteration decoding time at the $k$-th iteration on edge server $s$ within time slot $t$ is given by
\begin{equation}
\tau_s(k,t)=
\begin{cases}
\tau_s^{\rm I}(k,t), & 1\le k\le\tilde{k}_s(t),\\
\tau_s^{\rm II}(k,t), & \tilde{k}_s(t)<k\le K_s(t).
\end{cases}
\end{equation}
In~\eqref{load}, the numerator accumulates the memory occupied at each decoding iteration weighted by its duration, measuring the cumulative KV residency incurred by the inference engine, 
while the denominator represents the maximum available memory-time capacity
of the edge server within a slot.
By jointly integrating the spatial and temporal dimensions of KV cache occupation, $\eta_s(t)$ effectively captures the KV-centric workload pressure, particularly across heterogeneous edge servers. 
In the spatial dimension, normalizing the occupation by the total capacity $M_s$ ensures that the same KV cache size imposes a much heavier load on edge servers with scarce memory. 
In the temporal dimension, under an identical KV cache state, edge servers with weaker processing capabilities exhibit longer per-iteration decoding times $\tau_s(k,t)$, incurring a larger value of the memory-time consumption term in our metric.
This dual sensitivity allows $\eta_s(t)$ to characterize the heterogeneous load distribution, providing a tractable and interpretable indicator of LLM inference workload.

\subsection{Problem Formulation}\label{S-E}
In summary, the end-to-end latency for request $i$ scheduled to edge server $s$ in time slot $t$, consisting of the transmission, batch waiting, prefill, and decoding latency, is formulated as  
\begin{equation}
    D^{\text{tot}}_{i,s}(t) = D^{\text{tx}}_{i,s}(t) + D^{\text{wa}}_{i,s}(t) + D^{\text{pre}}_{s}(t) + D_{i,s}^{\text{dec}}(t).
\end{equation}

Beyond latency, load balance across edge servers is also considered to maintain the workload deviation within acceptable bounds.
To ensure long-term load balance, the variance-based workload deviation of each edge server $s\in\mathcal{S}$ from the system average
must satisfy the following constraint
\begin{equation}\label{c4}
    \lim_{T \to \infty} \frac{1}{T} \sum_{t=1}^{T} \big(\eta_s(t) - \bar{\eta}(t)\big)^2  \le \epsilon, 
    \ \ \forall s \in \mathcal{S},
\end{equation}
where $\bar{\eta}(t) = \frac{1}{S} \sum_{s \in \mathcal{S}} \eta_s(t)$ denotes the average workload over all edge servers in time slot $t$ and 
$\epsilon$ denotes the predefined long-term tolerance threshold.

With the objective of minimizing the long-term time-average end-to-end latency of all inference requests while ensuring load balance across edge servers, we formulate an online optimization problem over the scheduling decision $\boldsymbol{x}(t) = \{x_{i,s}(t) \mid \forall i\in \mathcal{Q}(t), \forall s\in \mathcal{S}\}$ for each time slot $t\in \mathcal{T}$. 
The problem is formally stated as
%
\begin{subequations}\label{opt}
\begin{align}
    \mathbf{P1}: & \min_{\{\boldsymbol{x}(t)\}}  \lim_{T\rightarrow \infty} \frac{1}{T} \sum_{t=1}^T \mathbb{E}\Big[ \sum_{i\in \mathcal{Q}(t)}\sum_{s\in\mathcal{S}}  x_{i,s}(t)  D^{\text{tot}}_{i,s}(t) \Big], \label{obj} \\
    \text{s.t.} \ & \eqref{c4}, \nonumber\\
    & M_s^{\text{pk}}(t) \le M_s, \ \ \forall s \in \mathcal{S},\ \forall t\in\mathcal{T}, \label{c1}\\
    & x_{i,s}(t)\in\{0,1\}, \ \ \forall i\in\mathcal{Q}(t),\ \forall s \in \mathcal{S},\ \forall t\in\mathcal{T}, \label{c2}\\
    & \sum_{s \in \mathcal{S}} x_{i,s}(t) = 1, \ \ \forall i \in \mathcal{Q}(t),\ \forall t \in \mathcal{T},  \label{c3}
\end{align}
\end{subequations}
where constraint~\eqref{c4} specifies the long-term average load-balancing constraint.
Constraint~\eqref{c1} ensures that the KV cache occupation at each edge server does not exceed its available memory capacity, guaranteeing the feasibility of LLM inference execution. 
Constraints~\eqref{c2} and~\eqref{c3} enforce the binary nature of the scheduling variables and the unique assignment of each request to an edge server, respectively.

Solving problem $\mathbf{P1}$ directly is nontrivial due to the following complexities.
Fundamentally, the discrete scheduling decisions (i.e., $\boldsymbol{x}(t)$) make the problem a nonlinear integer programming problem, which is known to be NP-hard.
In addition, the attributes of inference requests (i.e., input/output token lengths) and the locations of mobile users vary over time in an a priori unknown manner, resulting in fluctuating inference workloads and time-varying channel conditions. Meanwhile, the heterogeneous prefill and decoding throughputs across edge servers further complicate the overall service latency.
Last but not least, the scheduling decisions are coupled across time slots. 
The end-to-end latency of a request (i.e., $D^{\text{tot}}_{i,s}(t)$) depends not only on its own scheduling decision in time slot $t$, but also on the system state resulting from past scheduling decisions and the future system evolution driven by subsequent scheduling decisions.
This temporal coupling, compounded by the fact that $D^{\text{tot}}_{i,s}(t)$ only becomes observable after request completion, introduces strong interdependence between decisions across time, making the optimization problem particularly challenging.

\section{LYREO Approach}\label{alg}
In this section, we develop LYREO, a novel approach that treats $\mathbf{P1}$ as a sequential decision-making problem and jointly addresses the challenges identified above. 
Specifically, 
we first reformulate problem $\mathbf{P1}$ by leveraging Lyapunov optimization to handle the long-term load-balancing constraint. 
We then introduce a reward redistribution mechanism to attribute delayed latency feedback to its associated decisions before optimizing the scheduling policy.
Finally, we provide the theoretical analysis of the proposed approach.


\subsection{Problem Reformulation via Lyapunov Optimization}
Constraint~\eqref{c4} couples the scheduling decisions over the entire time horizon and therefore cannot be evaluated from the current time slot alone. To expose the accumulated load imbalance to each scheduling decision, we define a virtual queue for the workload deviation of each edge server, whose dynamics evolve as 
\begin{equation}\label{queue}
    Z_s(t+1) = \max \{Z_s(t) + Y_s(t) - \epsilon, 0\}, \ \ s \in \mathcal{S}, 
\end{equation}
where $Y_s(t) = (\eta_s(t) - \bar{\eta}(t)) ^2$ denotes the instantaneous workload deviation, and $Z_s(t)$ denotes the length of the virtual queue (with $Z_s(0) = 0$), tracking the cumulative amount by which $Y_s(t)$ exceeds the threshold $\epsilon$. 
Queue stability implies that the long-term workload deviation does not exceed $\epsilon$, thereby satisfying constraint~\eqref{c4}.

Collecting the virtual queues of all edge servers, we define $\boldsymbol{Z}(t) = [Z_1(t), Z_2(t), \dots, Z_S(t)]$ as the queue backlog vector in time slot $t$.
To characterize how well the scheduling adheres to long-term load balancing, we define the following quadratic Lyapunov function~\cite{neely10morgan} 
\begin{equation}
    \mathcal{L}(\boldsymbol{Z}(t)) \triangleq \frac{1}{2}\sum_{s\in\mathcal{S}} Z_s(t)^2. 
\end{equation}
A small value of $\mathcal{L}(\boldsymbol{Z}(t))$ indicates that all virtual queues are close to zero, i.e., the long-term constraint is well satisfied, and thus the system should aim to keep $\mathcal{L}(\boldsymbol{Z}(t))$ small.
Then, the conditional Lyapunov drift is given by
\begin{equation}\label{drift}
\Delta(\boldsymbol{Z}(t)) \triangleq \mathbb{E}\left[\mathcal{L}(\boldsymbol{Z}(t+1)) - \mathcal{L}(\boldsymbol{Z}(t)) \mid \boldsymbol{Z}(t)\right],
\end{equation}
which represents the expected change in the Lyapunov function over time slots, and a smaller drift indicates a more stable queue. 
However, according to~\eqref{drift}, the Lyapunov drift still depends on the system state in the next time slot, making it intractable to compute directly.
To avoid relying on future system information, we derive an upper bound on the Lyapunov drift, as provided in Lemma~\ref{driftbound}.
\begin{lemma}\label{driftbound}
    The Lyapunov drift $\Delta(\boldsymbol{Z} (t))$ is  upper bounded by
    \begin{equation}\label{eq:driftbound}
        \Delta(\boldsymbol{Z}(t)) \le \Gamma + \mathbb{E} \Big[ \sum_{s\in\mathcal{S}} Z_s(t) \left( Y_s(t) - \epsilon \right) \mid \boldsymbol{Z}(t) \Big],
    \end{equation}
    where $\Gamma \geq \left( \sum_{s\in \mathcal{S}} \mathbb{E} [(Y_s(t)-\epsilon)^2]\right)/2$ is a positive constant that bounds the expected squared terms over all time slots.    
\end{lemma}

\begin{proof}
    Squaring both sides of the virtual queue update in~\eqref{queue} and using the fact that $\{[x]^+\}^2\le x^2$ for any $x$ gives
    \begin{equation}\label{l1e1}
    \begin{aligned}
        Z_s(t+1)^2 & = \left\{ [Z_s(t) + Y_s(t) - \epsilon]^+\right\}^2 \\
        &\le (Z_s(t) + Y_s(t) - \epsilon)^2, \ \ s \in \mathcal{S}.
    \end{aligned}
    \end{equation}
    
    Summing~\eqref{l1e1} over all $s \in \mathcal{S}$ and dividing by $2$ yields
    \begin{equation}
    \begin{aligned}
        \frac{1}{2} \sum_{s\in \mathcal{S}} Z_s(t+1)^2 
        & \le \frac{1}{2} \sum_{s\in \mathcal{S}}  Z_s(t)^2 + \frac{1}{2}  \sum_{s\in \mathcal{S}} (Y_s(t) - \epsilon)^2  \\ 
        & \quad +  \sum_{s\in \mathcal{S}} Z_s(t) (Y_s(t) - \epsilon).
    \end{aligned}
    \end{equation}
    
    Subtracting $\left(\sum_{s\in\mathcal{S}} Z_s(t)^2\right)/2$ from both sides and incorporating the conditional expectation with respect to $\boldsymbol{Z}(t)$, the upper bound for $\Delta(\boldsymbol{Z} (t))$ can be derived as    
    \begin{equation}
    \begin{aligned}
        \Delta(\boldsymbol{Z} & (t))  = \mathbb{E} \Big[ \frac{1}{2}\sum_{s\in \mathcal{S}}Z_s(t+1)^2 - \frac{1}{2}\sum_{s\in \mathcal{S}}Z_s(t)^2  \mid\boldsymbol{Z}(t) \Big]\\
        &\le \mathbb{E} \Big[  \frac{1}{2} \sum_{s\in \mathcal{S}}   (Y_s(t) - \epsilon)^2 + \sum_{s\in \mathcal{S}} Z_s(t) (Y_s(t) - \epsilon)\mid\boldsymbol{Z}(t) \Big].
    \end{aligned}
    \end{equation}    
    Substituting the definition of $\Gamma$ into the preceding inequality yields~\eqref{eq:driftbound}, which completes the proof.
\end{proof}

Lemma~\ref{driftbound} provides an upper bound on the Lyapunov drift that no longer explicitly depends on the future virtual-queue state.
%
Based on the Lyapunov optimization framework, the objective function of the original problem  $\mathbf{P1}$ 
can be rewritten as the minimization of the following drift-plus-penalty function, 
\begin{equation}\label{driftpp}
    \begin{aligned}
        \Delta(\boldsymbol{Z} & (t)) + \nu\mathbb{E}\Big[ \sum_{i\in \mathcal{Q}(t)}\sum_{s\in\mathcal{S}}  x_{i,s}(t)  D^{\text{tot}}_{i,s}(t) \mid \boldsymbol{Z} (t) \Big]\\
        &\le \Gamma + \mathbb{E} \Big[ \sum_{s\in\mathcal{S}} Z_s(t) \left( Y_s(t) - \epsilon \right) \mid \boldsymbol{Z}(t) \Big] \\ 
        & \quad + \nu \mathbb{E}\Big[ \sum_{i\in \mathcal{Q}(t)}\sum_{s\in\mathcal{S}}  x_{i,s}(t)  D^{\text{tot}}_{i,s}(t) \mid \boldsymbol{Z} (t) \Big],
    \end{aligned}
\end{equation}
where $\nu>0$ is a control parameter that balances the trade-off between minimizing end-to-end latency and maintaining long-term load balance.
By dropping terms independent of $\boldsymbol{x}(t)$,
we obtain the following per-slot reformulation of problem $\mathbf{P1}$,
\begin{equation}\label{opt2}
\begin{aligned}
    \mathbf{P2}:\ \min_{\boldsymbol{x}(t)} \  \mathbb{E}&\Big[ \nu \sum_{i\in \mathcal{Q}(t)} \sum_{s\in\mathcal{S}} x_{i,s}(t) D^{\text{tot}}_{i,s}(t) \\
    & + \sum_{s\in\mathcal{S}} Z_s(t)Y_s(t) \mid \boldsymbol{Z}(t) \Big], \\
    \text{s.t.} \ & \eqref{c1}, \ \eqref{c2}, \ \eqref{c3}.
\end{aligned}
\end{equation}

Note that, in problem $\mathbf{P2}$, the long-term load-balancing constraint of $\mathbf{P1}$ has been transformed into a per-slot corrective term.
However, $\mathbf{P2}$ remains intractable, mainly because the latency term  $D^{\text{tot}}_{i,s}(t)$ in the objective depends on future system information.
Specifically, the scheduling decision must be committed in the current time slot, whereas the resulting end-to-end latency can only be observed after inference completion, and is further influenced by the arrival and scheduling of subsequent requests assigned to the same edge server. 
To this end, in the next subsection, we propose a novel policy learning framework to efficiently address this challenge.

\subsection{Delayed-Feedback-Aware Policy Learning}
\subsubsection{Sequence-Markov Decision Process Formulation}
Problem $\mathbf{P2}$ constitutes a sequential decision-making process, as each scheduling action updates the system states encountered by subsequent decisions.
Accordingly, we formulate it as a sequence-Markov decision process (SDP), represented by the tuple $\langle \mathbf{S}, \mathbf{A}, \mathbf{P}, \mathbf{R}, \gamma \rangle$, where, unlike a standard Markov decision process, the rewards are not required to satisfy the Markov property.
At each time step $t$, the scheduler observes the current state $\boldsymbol{s}_t \in \mathbf{S}$ and takes an action $\boldsymbol{a}_t \in \mathbf{A}$, which transitions the environment to a new state $\boldsymbol{s}_{t+1}$ with probability $\Pr[\boldsymbol{s}_{t+1}\mid \boldsymbol{s}_t, \boldsymbol{a}_t]$ and returns a reward $r_{t+1} \in \mathbf{R}$.
The discount factor $\gamma \in (0, 1)$ balances the immediate and future rewards. 
The detailed definitions of the state and action space, and reward function are provided as follows.

$\bullet$ \textit{State Space.} 
The state represents the system information available to the scheduler in each time slot $t$, defined as 
\begin{equation}
        \boldsymbol{s}_t = \{ \boldsymbol{s}^{\text{req}}_t, \boldsymbol{s}^{\text{ser}}_t \},
\end{equation}
where $\boldsymbol{s}^{\text{req}}_t = [l_i^{\text{in}}, \hat{l}_i^{\text{out}}, \boldsymbol{o}_i]_{i\in \mathcal{Q}(t)}$ denotes the attributes of the inference requests to be scheduled, 
and  $\boldsymbol{s}^{\text{ser}}_t = [B_s, v_s^{\text{pre}}, v_s^{\text{dec}}(c), \vert\mathcal{Q}^{\text{bg}}_s(t)\vert, \eta_s(t-1), Z_s(t)]_{s \in \mathcal{S}}$ 
denotes the server-side state information, including the attributes (i.e., bandwidth, prefill and decoding throughput) and the dynamic status (i.e., the number of backlogged requests,  the recent workload, and the virtual-queue length of load-balancing deviation) of each edge server.
Here, $v_s^{\text{dec}}(c)$ denotes the decoding throughput evaluated at a constant reference occupation $c$, serving as an indicator of the server's decoding capability.

$\bullet$ \textit{Action Space.} 
In each time slot $t$, the scheduler determines the scheduling decision for all requests in $\mathcal{Q}(t)$. 
The action is defined as $\boldsymbol{a}_t = \boldsymbol{x}(t) = \{x_{i,s}(t) \mid \forall i\in \mathcal{Q}(t), \forall s\in \mathcal{S}\} $.

$\bullet$ \textit{Reward Function.} 
By taking action $\boldsymbol{a}_t$ under state $\boldsymbol{s}_t$, the scheduler receives a numerical reward defined according to the objective function in~\eqref{opt2}, given by
\begin{equation}\label{rewardfunc}
        r_{t+1} =  -\nu \sum_{s\in\mathcal{S}} \sum_{i\in \mathcal{Q}'_s(t)} D^{\text{tot}}_{i,s}(t) - \sum_{s\in\mathcal{S}} Z_s(t)Y_s(t)- \varrho \Upsilon(t),
\end{equation}
where $\mathcal{Q}'_s(t)$ represents the set of requests that complete their inference on edge server $s$ in time slot $t$. 
$\Upsilon(t) = \sum_{s\in\mathcal{S}}[M_s^{\text{pk}}(t) - M_s]^+$ is a penalty term enforcing the GPU memory constraint, and $\varrho > 0$ is the corresponding penalty coefficient. 
Overall, the reward jointly incorporates the total latency of requests, the load-balancing deviation captured by the virtual queue, and the memory feasibility penalty. 

\subsubsection{Reward Redistribution}
As discussed above, the latency term in~\eqref{rewardfunc} introduces delayed rewards into the framework, 
so the scheduler cannot immediately evaluate which earlier assignment caused the eventual outcome.
Inspired by~\cite{arjona19neurips}, we introduce reward redistribution to recover a decision-level reward signal.
Reward redistribution is a procedure for an SDP that redistributes the total return $\sum_{t=0}^{T}r_{t+1}$ over the sequence of state-action pairs $( \boldsymbol{s}_0, \boldsymbol{a}_0, \dots, \boldsymbol{s}_T, \boldsymbol{a}_T )$.
Since reward redistribution preserves the cumulative return of the sequence, the expected return under any policy remains unchanged. 
Therefore, the original SDP and the redistributed SDP share the same optimal policy, indicating that reward redistribution does not alter the underlying optimization objective.

Following the optimal second-order Markov reward redistribution  in~\cite{arjona19neurips,chen21tcom, gui26tnse}, 
we consider a reward signal that satisfies
\begin{equation}\label{optrr}
        \mathbb{E}[r'_{t+1} \mid \chi_{t-1}, \chi_t ]  = 
        q^{\pi}(\chi_t) - q^{\pi}(\chi_{t-1}),
\end{equation}
where $\chi_t=(\boldsymbol{s}_t,\boldsymbol{a}_t)$ and $q^{\pi}(\chi_t)$ denote the  state-action pair in time slot $t$ and its corresponding $Q$-value under policy $\pi$, respectively.
Since the $Q$-value represents the expected cumulative return starting from a given state-action pair, Eq.~\eqref{optrr} implies that the redistributed reward mathematically captures the increment in the expected return brought by the current state-action pair.
By immediately assigning this increment to that pair, the expected future redistributed reward becomes zero, which eliminates reward delay in expectation. 
Consequently, this formulation provides a step-wise objective for policy learning.
To illustrate this intuitively, if a state-action pair increases the cumulative return, i.e., $q^{\pi}(\chi_{t}) > q^{\pi}(\chi_{t-1})$, the corresponding redistributed reward $r'_{t+1}$ becomes positive, even though this improvement may not be reflected immediately in the original delayed reward.

Although the second-order Markov reward redistribution is theoretically optimal for finite-horizon problems, it relies on predicting the cumulative return associated with each observed state-action sequence.
In finite-horizon settings, the complete sequence return naturally serves as the supervision target for this prediction.
However, under an infinite-horizon formulation, the complete sequence return is unavailable.
To overcome this limitation, we develop a novel infinite-horizon return prediction framework by integrating sequence truncation with value bootstrapping.
Specifically, instead of using the complete sequence return as the supervision target, we construct a return target that combines the actual rewards observed up to a truncation point $H$ with a bootstrapped estimate for the remaining horizon.
We define this return target as 
\begin{equation}\label{ytarget}
    \hat{y}_H \triangleq  \sum_{h=0}^{H-1}\gamma^{h}r_{h+1} + \gamma^{H} V_\phi(\boldsymbol{s}_{H}),
\end{equation}
where the first term is the cumulative discounted reward over the first $H$ steps, i.e., the original infinite-horizon trajectory is truncated after $H$ steps, 
while the second term is a bootstrap estimate of the remaining discounted return, provided by the value network introduced later.

Since this return prediction requires estimating the expected return conditioned on variable-length state-action sequences,
we employ an LSTM network, parameterized by $\psi$ and denoted by $g_{\psi}(\cdot)$, which is well-suited for modeling such sequential dependencies.
Given the sequence observed up to time slot $t$, the network outputs a prediction
\begin{equation}\label{outpre}
g_t \triangleq g_\psi(\chi_{0:t}) \approx \mathbb{E}[\hat{y}_H \mid \chi_0,\ldots,\chi_t].
\end{equation}
The network $g_\psi(\cdot)$ is then trained in a supervised manner by minimizing the 
mean-squared error (MSE) between $g_t$ and the truncated-bootstrapped target 
$\hat{y}_H$ in~\eqref{ytarget}.
After training, the LSTM network serves as a return predictor that estimates the expected cumulative return from any observed state-action sequence. 
Given the sequence observed up to time slot $t$, the LSTM network outputs the corresponding prediction $g_t$.
The redistributed reward is then computed as
\begin{equation}\label{redist.rew}
    r'_{t+1} = g_{t} - g_{t-1},
\end{equation}
which measures the increment in the predicted return brought by the current state-action pair. 
Therefore, rewards are reassigned from delayed outcomes to the state-action pairs that contribute to the eventual return, serving as a practical approximation of the theoretical reward redistribution in~\eqref{optrr}.
\subsubsection{Scheduling Policy Optimization}
The redistributed reward provides an immediate signal for learning how each observed system state should be mapped to request assignments. 
Accordingly, we parameterize the scheduling policy $\pi_{\boldsymbol{\theta}}(\boldsymbol{a}\mid\boldsymbol{s})$ using a neural network with parameters $\boldsymbol{\theta}$, which outputs a categorical distribution over the edge servers for each incoming request.
Also, we define a value network parameterized by $\phi$ to estimate the state value function $V_{\phi}(\boldsymbol{s})$.
Following proximal policy optimization (PPO)~\cite{schulman17arxiv}, we optimize these two networks as the actor and critic, respectively, eliminating the need for an explicit system transition model. 
To ensure stable policy updates, a surrogate objective function is utilized to constrain the optimization step by clipping the probability ratio between the new and old policies.
Let 
$\rho_t(\theta) = \pi_{\theta}(\boldsymbol{a}_t|\boldsymbol{s}_t) / \pi_{\theta_{\text{old}}}(\boldsymbol{a}_t|\boldsymbol{s}_t)$ 
denote the probability ratio between the new and old policy.
The surrogate objective function is defined as
\begin{equation}\label{aupdate}
		L^{\text{CLIP}}(\theta) =  \mathbb{E}_t\big[\min  \big(\rho_t(\theta) \hat{A}_t,
		\text{clip} (\rho_t(\theta) , 1 - \varepsilon, 1 + \varepsilon) \hat{A}_t\big)\big],
\end{equation}
where $\varepsilon$ is a clipping hyperparameter. $\hat{A}_t$ represents the advantage function, which quantifies the relative benefit of taking action $\boldsymbol a_t$ over the policy average. 
We apply generalized advantage estimation (GAE), given by $\hat{A}_t = \sum_{l=0}^{\infty} (\gamma\lambda)^l \delta_{t+l}$, where $\delta_{t} = r'_{t+1} + \gamma V_{\phi}(\boldsymbol{s}_{t+1}) - V_{\phi}(\boldsymbol{s}_{t})$ denotes the temporal-difference (TD) error, and $\lambda\in[0,1]$ is a parameter of GAE.
Accordingly, the actor network is updated by maximizing the surrogate objective function $L^{CLIP}(\theta)$ using gradient ascent.

Subsequently, the value network is trained by minimizing the MSE between the estimated state value and target value,
\begin{equation}\label{cupdate}
L^{\text{VF}}(\phi) = \mathbb{E}_t[(V_{\phi}(\boldsymbol{s}_t) - \hat{V}_t)^2],
\end{equation}
where $\hat{V}_t = \hat{A}_t + V_{\phi_{\text{old}}}(\boldsymbol{s}_{t})$ is the target value.
The critic parameters $\phi$ are then updated via gradient descent on $L^{VF}(\phi)$.

\subsection{Algorithm Summary and Theoretical Analysis}
The primary steps of LYREO are summarized in Algorithm~\ref{alg:lyrappo}.
\begin{algorithm}[t]
\caption{LYREO Approach}
\label{alg:lyrappo}
\KwIn{Truncation length $H$; actor learning rate $\xi_A$; critic learning rate $\xi_C$; clipping parameter $\varepsilon$; discount factor $\gamma$; GAE parameter $\lambda$; PPO epochs $\Xi$}
\KwOut{Learned scheduling policy $\pi_\theta$}
Initialize actor network $\pi$, critic network $V$, and return predictor $g$ with random parameters $\theta, \phi, \psi$\;
\While{not converged}{
    Collect a batch of trajectories $\mathfrak{B} = \{(\boldsymbol{s}_t,\boldsymbol{a}_t,r_{t+1},\boldsymbol{s}_{t+1})\}_{t=0}^{H-1}$ using policy $\pi_\theta$\;
    \tcc{\small Reward Redistribution Learning}
    \ForEach{trajectory in $\mathfrak{B}$}{
        Compute learning target $\hat{y}_H$ via~\eqref{ytarget}\;
        Construct training pairs $\{(\chi_{0:t}, \hat{y}_H)\}_{t=0}^{H-1}$\;
    }
    Update return predictor $g_\psi$ by minimizing MSE between $g_\psi(\chi_{0:t})$ and $\hat{y}_H$ over all training pairs\;
    \tcc{\small Scheduling Policy Optimization}
    Freeze return predictor $g_\psi$\;
    \ForEach{trajectory in $\mathfrak{B}$}{
        Compute redistributed reward $r'_{t+1}$ via~\eqref{redist.rew}\;
        Compute TD error $\delta_t$ and advantage $\hat{A}_t$\;
    }
    Set $\theta_{\text{old}} \leftarrow \theta$\;
    \For{epoch $=1$ \KwTo $\Xi$}{
        Update actor parameters $\theta$ by maximizing $L^{\text{CLIP}}(\theta)$ via~\eqref{aupdate} with learning rate $\xi_A$\;
        Update critic parameters $\phi$ by minimizing $L^{\text{VF}}(\phi)$ via~\eqref{cupdate} with learning rate $\xi_C$\;
    }
}
\end{algorithm}
We now establish the theoretical performance guarantee of the proposed LYREO approach.
\begin{theorem}
\label{policy_gap}
Let $\pi^*$ denote an optimal policy of the delayed-reward SDP induced by problem $\mathbf{P2}$.
Let $\hat{\pi}$ denote the optimal policy learned by LYREO under the discount factor $\gamma$ and truncation length $H$.
Assume that the reward function is bounded by $|r_t| \leq r_{\max}$, 
the value estimation error of the critic network is bounded by $|V^{\pi}(\boldsymbol{s}) - V_{\phi}(\boldsymbol{s}) | \leq \epsilon_{\mathrm{v}}$ for any policy $\pi$, 
the conditional-mean regression error of the LSTM-based return predictor is bounded by $|g_t-\mathbb{E}[ \hat{y}_H\mid\chi_{0:t}]|\leq\epsilon_{\mathrm{r}}$,
and the policy optimization error of PPO is bounded by $\epsilon_{\mathrm{PPO}}$.
The performance gap of the learned policy is bounded by
\begin{equation}\label{gap}
    \begin{aligned}
        J(\pi^*) - J(\hat{\pi}) & \leq \frac{4}{1-\gamma} \left( \gamma^H \epsilon_{\mathrm{v}} + \frac{2\gamma^H}{1-\gamma}r_{\max} + \epsilon_{\mathrm{r}} \right) + \epsilon_{\mathrm{PPO}}\\
        & = \mathcal{O}\left( \frac{\gamma^H \epsilon_{\mathrm{v}}+ \epsilon_{\mathrm{r}}}{1 - \gamma} + \frac{\gamma^H r_{\max}}{(1 - \gamma)^2} +\epsilon_{\mathrm{PPO}}\right).
    \end{aligned}
\end{equation}
\end{theorem}

\begin{proof}
Let $r'^*$ denote the ideal untruncated redistributed reward derived from the true expected return under the optimal policy $\pi^*$. 
Based on the optimality of the second-order Markov reward redistribution established in~\cite{arjona19neurips}, we consider the SDP guided by $r'^*$ as the ideal objective that provides immediate step-wise reward signals.  
The performance gap introduced by LYREO arises from approximating these ideal signals via sequence truncation and value bootstrapping in policy learning.
We establish the bound in the following four steps.

\textbf{Step I: Bounding the Return Prediction Error.} 
Let $y_{\infty} = \sum_{h=0}^{\infty} \gamma^h r_{h+1}$ denote the true infinite-horizon cumulative return, and $\hat{y}_H$ denote the approximated return constructed via truncation and bootstrapping, given by~\eqref{ytarget}. The estimation error can be derived as 
\begin{equation}
    \begin{aligned}
        |y_{\infty} - \hat{y}_H| 
        &= \left| \sum_{h=H}^{\infty} \gamma^h r_{h+1} - \gamma^H V_{\phi}(\boldsymbol{s}_{H}) \right|.
    \end{aligned}
\end{equation}
By adding and subtracting the same term $\gamma^{H}V^{\pi}(\boldsymbol{s}_{H})$ and applying the triangle inequality, we obtain
\begin{equation}
    \begin{aligned}
        |y_{\infty} - \hat{y}_H| & \leq \gamma^H |V^{\pi}(\boldsymbol{s}_{H}) - V_{\phi}(\boldsymbol{s}_{H})| \\
        & \quad + \left| \sum_{h=H}^{\infty} \gamma^h r_{h+1} - \gamma^H V^{\pi}(\boldsymbol{s}_{H}) \right|.
    \end{aligned}
\end{equation}
By assumption, the first term is bounded by $\gamma^H \epsilon_{\text{v}}$. 
For the second term, by the definition of the value function, multiplying both sides by $\gamma^H$ and re-indexing the summation, we have 
$\gamma^H V^{\pi}(\boldsymbol{s}_{H}) = \mathbb{E}\left[ \sum_{h=H}^{\infty} \gamma^h r_{h+1} \mid \boldsymbol{s}_{H}\right].$
Since the instantaneous reward is bounded by $|r_t| \leq r_{\max}$, the absolute deviation between any realization of the truncated return and its expectation is bounded by $2\gamma^H r_{\max}/(1-\gamma)$. 
Combining the two bounds, the return prediction error satisfies
\begin{equation}
    |y_{\infty} - \hat{y}_H| \leq \gamma^H \epsilon_{\text{v}} + \frac{2\gamma^H}{1-\gamma}r_{\max} \triangleq \Pi.
\end{equation}

\textbf{Step II: Bounding the Redistributed Reward Error.} 
The LSTM network predicts the expected return conditioned on the state-action sequence $\chi_{0:t}$. 
Let $g^*_t = \mathbb{E}[y_{\infty} \mid \chi_{0:t}]$ and $\hat{g}_t = \mathbb{E}[\hat{y}_H \mid \chi_{0:t}]$ denote the ideal prediction under the true return and the ideal conditional expectation of the truncated-bootstrapped target, respectively.
Since $|y_{\infty} - \hat{y}_H| \leq \Pi$, Jensen's inequality implies 
\begin{equation}\label{44}
|g^*_t- \hat{g}_t| \leq \left|\mathbb E[y_{\infty}-\hat y_H\mid\chi_{0:t}]\right| \le\mathbb E\!\left[|y_{\infty}-\hat y_H|\mid\chi_{0:t}\right]\le\Pi.
\end{equation}
By assumption, the regression error between $g_t$, trained via~\eqref{outpre}, and the conditional expectation $\hat{g}_t$ is bounded by  $\epsilon_{\text{r}}$.
Combining this with~\eqref{44} via the triangle inequality,
\begin{equation}
    |g^*_t- g_t|\le |g^*_t-\hat{g}_t| + |\hat{g}_t - g_t| \le \Pi + \epsilon_{\text{r}}.
\end{equation}
Recalling the redistributed reward defined in~\eqref{redist.rew}, the error in the redistributed reward is bounded by
\begin{equation}
         |r'^*_{t+1}- r'_{t+1}| \leq |g^*_t - g_t | + |g^*_{t-1} - g_{t-1} | \leq 2(\Pi+\epsilon_{\text{r}}),
\end{equation}
where $r'^*_{t+1}$ and $r'_{t+1}$ are the ideal redistributed reward and its approximation obtained via the proposed truncation and bootstrapping scheme, respectively.

\textbf{Step III: Bounding the Value Function Error.} For any policy $\pi$, the difference between its value functions evaluated under the ideal and approximated redistributed rewards can be accumulated over the infinite horizon as
\begin{equation}
    \begin{aligned}
        |V^\pi_{r'^*}(\boldsymbol{s}) - V^\pi_{r'}(\boldsymbol{s})| &= \left| \mathbb{E}_\pi \left[ \sum_{t=0}^{\infty} \gamma^t (r'^*_{t+1} -r'_{t+1} ) \right] \right| \\ 
        &\leq \sum_{t=0}^{\infty} \gamma^t 2(\Pi+\epsilon_{\text{r}}) = \frac{2(\Pi+\epsilon_{\text{r}})}{1-\gamma}.
    \end{aligned}
\end{equation}

\textbf{Step IV: Deriving the Policy Gap.} We now evaluate the suboptimality of policy $\hat{\pi}_{r'}$. 
Note that we have $V^{\pi^*}_{r'} - V^{\hat{\pi}_{r'}}_{r'} \leq \epsilon_{\text{PPO}}$, as $\hat{\pi}_{r'}$ approximates the maximizer of the value function under $r'$ with a bounded optimization error $\epsilon_{\text{PPO}}$.
Leveraging the standard value difference decomposition then gives
\begin{equation}
    \begin{aligned}
        & J_{r'^*}(\pi^*) - J_{r'^*}(\hat{\pi}_{r'}) \\
        \leq \ & |V^{\pi^*}_{r'^*} - V^{\pi^*}_{r'}| + \epsilon_{\text{PPO}} + |V^{\hat{\pi}_{r'}}_{r'} - V^{\hat{\pi}_{r'}}_{r'^*}|   \\
        \leq \ & \frac{2(\Pi+\epsilon_{\text{r}})}{1-\gamma} + \epsilon_{\text{PPO}} + \frac{2(\Pi+\epsilon_{\text{r}})}{1-\gamma} = \frac{4(\Pi+\epsilon_{\text{r}})}{1-\gamma}+ \epsilon_{\text{PPO}}.
    \end{aligned}
\end{equation}

Substituting the definition of $\Pi$ into the preceding inequality yields~\eqref{gap},
which completes the proof.
\end{proof}

\section{Performance Evaluation}\label{eval}
\subsection{Simulation Settings}
We consider a system consisting of heterogeneous edge servers. Each edge server deploys the Llama-3.2-1B model and is equipped with one of three GPU platforms: NVIDIA GeForce RTX 3080, RTX 3090, and RTX 4090.
Following the hardware specifications of these GPUs, the available GPU memory capacities are set to $M_s \in \{10, 20, 24\}$~GB, respectively. 
The prefill rate $v_s^{\text{pre}} \in \{2500, 4000, 8000\}$~tokens/s and the decoding throughput $v_s^{\text{dec}}(\cdot)$ are obtained via offline profiling of the Llama-3.2 model using a vLLM-based serving engine on the corresponding GPU platforms.
Specifically, $v_s^{\text{dec}}(\cdot)$ is modeled as a non-increasing function of the instantaneous KV cache occupation to reflect the memory-bound nature of the decoding phase.
Inference requests are drawn from the LMSYS-Chat-1M dataset~\cite{zheng24arxiv}, with the input and output token lengths
averaging $70$ and $215$ tokens, respectively. The KV cache memory occupation per token is set to $\alpha=16$~KB, while the transmission size of a single token is $\beta=16$~bits.
We set $T=800$ time slots, slot duration $\Delta=1$~s, and load-balancing tolerance $\epsilon=0.01$.
User locations evolve randomly within a $200 \times 200$~$\text{m}^2$ service area in each time slot. 
\begin{table}[!t]
\centering
\caption{Simulation Parameters}
\label{params}
\footnotesize
\setlength{\tabcolsep}{3pt}
\renewcommand{\arraystretch}{1.15}
\begin{tabular}{@{}ll | ll@{}}
\toprule
\textbf{Parameter} & \textbf{Value} & \textbf{Parameter} & \textbf{Value} \\
\midrule
Bandwidth $B_s$ & $[6,14]$\,MHz & Clipping parameter $\varepsilon$ & $0.2$ \\
Transmission power $p_i$ & $[15,25]$\,dBm & Learning rate & $3\times10^{-4}$ \\
Noise PSD $N_{0,s}$ & $-174$\,dBm/Hz & PPO training epochs & $10$ \\
Discount factor $\gamma$ & $0.98$ & Random seeds & $5$ \\
\bottomrule
\end{tabular}
\vspace{-0.1in}
\end{table}

During training, none of the algorithms has access to the complete trajectory. Instead, policy optimization relies only on the currently observed trajectory, consistent with the infinite-horizon formulation. After the $T$-th time slot, all unfinished inference requests continue execution until completion, and the rewards are accumulated and assigned to the final step. The neural networks are implemented using PyTorch~1.8. Both the actor and critic networks in the PPO architecture consist of four fully connected layers. The return predictor for reward redistribution is implemented as a two-layer LSTM with a hidden dimension of 64. The remaining parameters are summarized in Table~\ref{params}, following widely adopted settings from the literature~\cite{zhang25twc, mou26tsc}.

\subsection{Baseline Approaches}
To comprehensively evaluate the proposed LYREO approach, we consider three categories of baseline schemes.

First, two DRL-based baselines are implemented: the Lyapunov-assisted PPO algorithm (\textbf{Ly-PPO})~\cite{younesi25ucc}, which integrates Lyapunov optimization with PPO but excludes the proposed reward redistribution framework, and the conventional PPO algorithm (\textbf{PPO})~\cite{qiu25fcn}, which optimizes latency with the load-balancing constraint violation incorporated as a penalty term, without Lyapunov-based constraint transformation or reward redistribution framework.
Second, to validate the effectiveness of the DEED architecture, we implement a static batching scheme (\textbf{Static})~\cite{li25lcn}, where inference requests are grouped into fixed batches for decoding: newly arrived requests are not admitted into an ongoing decoding batch, and completed requests remain until the entire batch finishes decoding.
Finally, two heuristic scheduling strategies are included as classical non-learning baselines: 
Least-Loaded (\textbf{LL}), which schedules each request to the edge server with the lowest current workload, and
\textbf{Random}, which schedules each request to a randomly selected edge server.

\subsection{Performance Comparison and Analysis}

\subsubsection{Convergence Performance}
\begin{figure}[!t]
	\centering
	\includegraphics[width=1.01\linewidth]{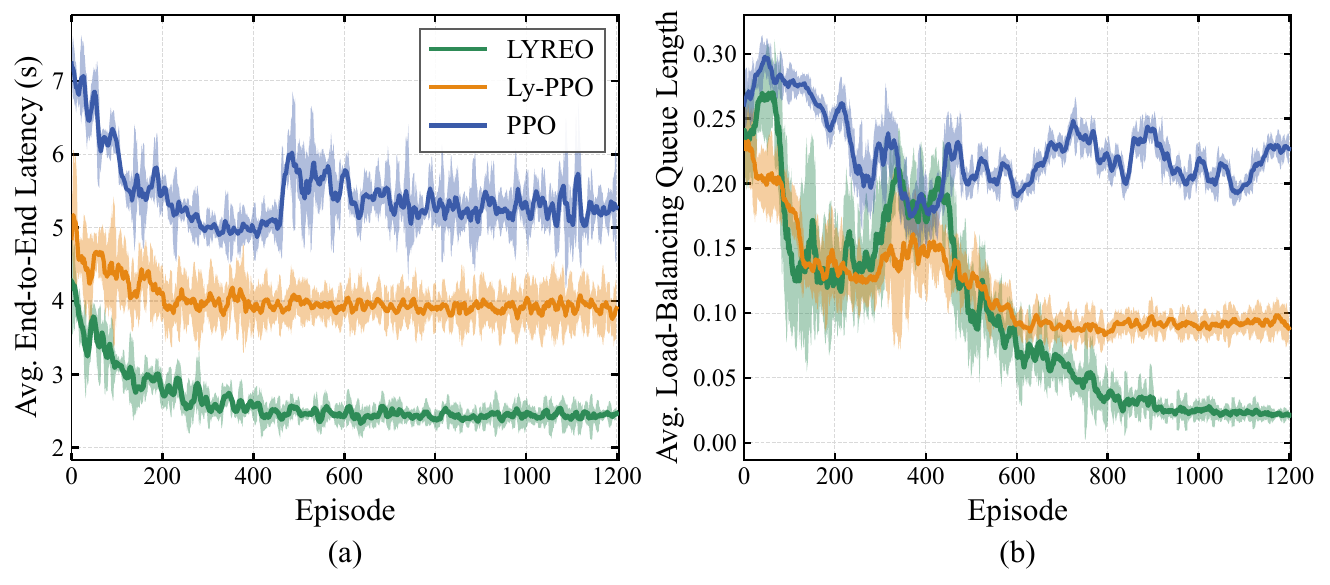}
	\caption{Convergence performance of LYREO, Ly-PPO, and PPO approaches. 
	}
	\label{converg}
    \vspace{-0.1in}
\end{figure}

Fig.~\ref{converg} presents the convergence behavior of the three learning-based approaches regarding average end-to-end latency and load-balancing queue length. 
Solid curves and shaded regions represent the mean and min-max range across seeds, respectively. 
As shown in Fig.~\ref{converg}(a), LYREO achieves the lowest end-to-end latency throughout the training process and converges to a stable value of approximately $2.4$~s. In comparison, Ly-PPO stabilizes at around $3.9$~s, while conventional PPO converges to a substantially higher latency of approximately $5.2$~s.
Fig.~\ref{converg}(b) shows a consistent performance advantage in load-balancing control. After convergence, LYREO reduces the average load-balancing queue length to approximately $0.02$, compared with about $0.09$ for Ly-PPO and $0.22$ for PPO (computed post hoc). 
Moreover, the two Lyapunov-guided approaches maintain relatively stable queue levels, whereas PPO remains at a higher and more variable plateau.
The improvement of LYREO over Ly-PPO highlights the benefit of reward redistribution. By attributing delayed inference outcomes to the scheduling decisions responsible for them, LYREO provides more informative temporal feedback for policy learning, leading to the best overall convergence performance.
This benefit extends beyond latency: since load balancing is incorporated into the drift-plus-penalty function, the same mechanism also helps the policy associate accumulated imbalance with the scheduling decisions, improving the load-balancing component.
Meanwhile, the gap between PPO and the two Lyapunov-guided approaches demonstrates the advantage of incorporating the Lyapunov framework into policy learning. Although PPO also penalizes excessive workload deviation, its balancing pressure does not adapt to the accumulated deviation over time. 
In contrast, the Lyapunov virtual queue provides an adaptive corrective signal, resulting in lower and more stable load-balancing queue lengths.\looseness=-1

\subsubsection{Distribution Comparison}
\begin{figure}[!tb]
	\centering
	\includegraphics[width=1.01\linewidth]{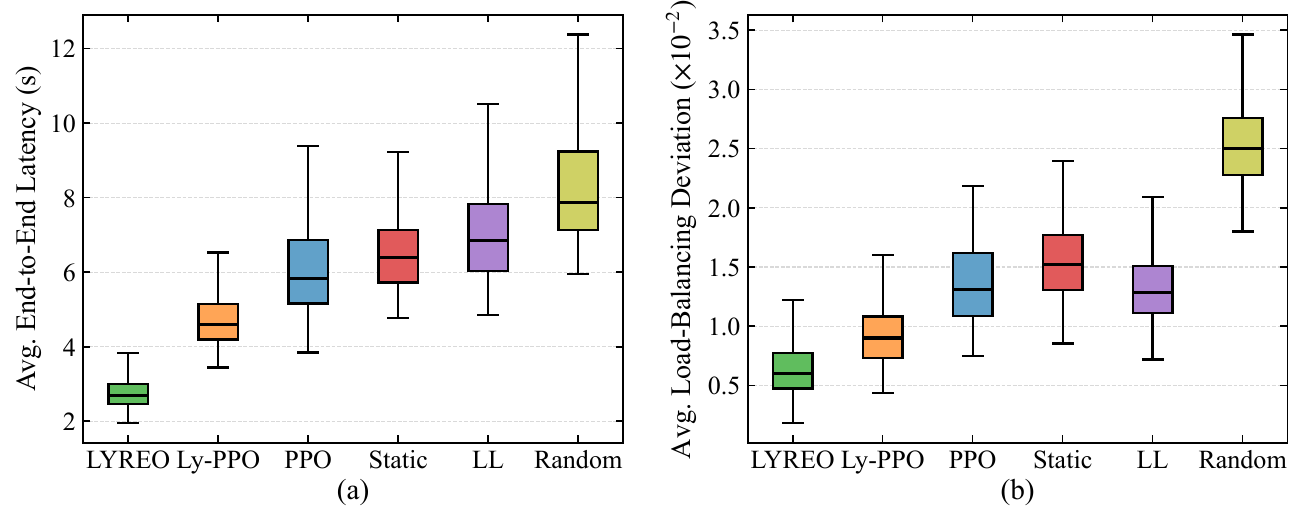}
    \caption{Boxplot comparison of different approaches.}
	\label{box}
\end{figure}

Fig.~\ref{box} examines the distributions of the two performance metrics
over the evaluation runs. 
In Fig.~\ref{box}(a), LYREO achieves the lowest median end-to-end latency and the most compact distribution, followed by Ly-PPO, whereas PPO, Static, LL, and Random exhibit higher medians and longer upper tails. 
Fig.~\ref{box}(b) shows that LYREO and Ly-PPO maintain lower load-balancing deviations than the other baselines. 
Overall, the consistent advantages of LYREO and Ly-PPO demonstrate the efficacy of handling load-balancing constraints through adaptive Lyapunov guidance.
On the other hand, PPO outperforms most non-learning baselines across the two metrics, indicating that the penalty term enables it to learn a meaningful trade-off between latency and load balancing.
However, treating load balancing as a fixed penalty limits its ability to adapt to time-varying server imbalance. Also, the delayed feedback associated with the penalty causes PPO to react only after requests have accumulated on particular edge servers, leading to worse performance than LYREO and Ly-PPO.
Among the non-learning baselines, Static achieves lower latency than LL, whereas LL provides better load balancing. Random performs worst on both metrics, exhibiting the highest medians and the widest distributions.

\subsection{Load-Balancing Ablation}
This subsection evaluates the necessity of long-term load balancing and validates the KV cache memory-time consumption workload through two variants of LYREO for the ablation study. 
Specifically, \textbf{LYREO-Count} replaces the proposed KV cache memory-time consumption with the number of active requests, while \textbf{LYREO-NoLB} removes the long-term load-balancing constraint.
Beyond the two primary optimization objectives, we introduce two additional evaluation metrics: the P99 end-to-end latency, defined as the $99$th percentile of request-level latency, and the high-KV ratio, defined as the percentage of time slots across all edge servers where the peak KV cache utilization exceeds $90\%$ of the reserved memory capacity limit. 
These two metrics, which are not directly optimized by the considered approaches, respectively characterize the worst-case, user-perceived latency that governs QoS satisfaction and the frequency of near-limit KV cache operation in edge servers, providing complementary evidence for the system-level performance of scheduling approaches.

As shown in Table~\ref{tab:lb_comparison}, LYREO-NoLB achieves the lowest average end-to-end latency, since it is free to concentrate inference requests on the servers offering the lowest latency. However, this latency advantage is accompanied by the cost of less balanced resource utilization.
Without the load-balancing mechanism, request concentration drives the average load-balancing deviation up by $125.0\%$ relative to LYREO, while the high-KV ratio rises to $14.35\%$, a $123.5\%$ relative increase. More importantly, this imbalance translates into substantially worse tail behavior, with the P99 end-to-end latency reaching $6.73$~s, $79.9\%$ higher than that of LYREO.
These results show that latency-oriented scheduling alone cannot prevent persistent workload concentration. Explicit long-term load balancing instead preserves KV cache headroom across edge servers and mitigates the tail-latency degradation, at the expected cost of higher average latency.
Meanwhile, LYREO-Count consistently underperforms LYREO across all four metrics, increasing the average latency to $2.79$~s, the P99 latency to $4.56$~s, the average load-balancing deviation to $0.011$, and the high-KV slot ratio to $8.04\%$.
This indicates that conventional count-based balancing does not balance the resource pressure imposed on heterogeneous edge servers, since inference requests differ in their KV cache occupation and duration. By jointly capturing occupation and residence time, the KV cache memory-time consumption provides a more reliable workload signal, validating its design for long-term load-balancing.
\begin{table}[t]
\caption{Load-Balancing and Workload Metrics Ablation}
\label{tab:lb_comparison}
\centering
\scriptsize
\setlength{\tabcolsep}{1.2pt}
\renewcommand{\arraystretch}{1.12}
\begin{tabular}{@{}lcccc@{}}
\toprule
\shortstack[b]{\textbf{Approach} \\ ~}
& \shortstack{\textbf{Avg. End-to-End}\\ \textbf{Latency (s)}}
& \shortstack{\textbf{Avg. Load-Balancing}\\ \textbf{Deviation ($\times 10^{-2}$)}}
& \shortstack{\textbf{P99 End-to-End}\\ \textbf{Latency (s)}}
& \shortstack{\textbf{High-KV}\\ \textbf{Ratio (\%)}} \\
\midrule
LYREO
& \shortstack{$2.63$\\ {\tiny (Ref.)}}
& \shortstack{$\mathbf{0.8}$\\ {\tiny (Ref.)}}
& \shortstack{$\mathbf{3.74}$\\ {\tiny (Ref.)}}
& \shortstack{$\mathbf{6.42}$\\ {\tiny (Ref.)}} \\
\addlinespace[2pt]
LYREO-Count
& \shortstack{$2.79$\\ {\tiny ($+6.1\%$)}}
& \shortstack{$1.1$\\ {\tiny ($+37.5\%$)}}
& \shortstack{$4.56$\\ {\tiny ($+21.9\%$)}}
& \shortstack{$8.04$\\ {\tiny ($+25.2\%$)}} \\
\addlinespace[2pt]
LYREO-NoLB
& \shortstack{$\mathbf{2.02}$\\ {\tiny ($-23.2\%$)}}
& \shortstack{$1.8$\\ {\tiny ($+125.0\%$)}}
& \shortstack{$6.73$\\ {\tiny ($+79.9\%$)}}
& \shortstack{$14.35$\\ {\tiny ($+123.5\%$)}} \\
\bottomrule
\end{tabular}
\vspace{-0.1in}
\end{table}

\subsubsection{Impact of System Parameters}
Fig.~\ref{fig:Q} presents the impact of the number of users on algorithm performance, with the number of edge servers fixed at $S = 8$. Since each user generates one request per time slot in our system, $U$ directly determines the inference request intensity within each time slot.
Fig.~\ref{fig:Q}(a) illustrates that the average end-to-end latency of all methods increases as $U$ grows from $20$ to $60$, since more users generate more requests within each time slot, which share the limited wireless bandwidth and result in larger prefill and decoding batches.
The latency of LYREO increases from about $2.3$~s to $3.3$~s, while the second-best Ly-PPO rises from about $2.9$~s to $5.6$~s.
On average, LYREO therefore reduces latency by roughly $33\%$ relative to Ly-PPO, with an even larger gap relative to PPO, Static, LL, and Random.
Fig.~\ref{fig:Q}(b) shows that heavier traffic also increases the server-averaged load-balancing deviation of all schemes. Nevertheless, LYREO keeps the deviation below approximately $0.009$ for the entire tested range, satisfying the predefined tolerance on average.
Overall, LYREO consistently achieves the best performance across both metrics, and its advantage widens significantly under high request intensity.
This is because heavier request intensity prolongs request completion, leading to more delayed rewards, a setting in which reward redistribution is particularly effective and consequently amplifies the advantage of LYREO over Ly-PPO.
Meanwhile, higher intensity pushes the system closer to its capacity limit, where minor allocation imbalances can rapidly accumulate into substantial workload disparities.
The Lyapunov virtual queue resolves this by proactively regulating workload imbalance, allowing LYREO to yield its advantage over LL, PPO, and other baselines.

\begin{figure}[!tb]
	\centering
	\includegraphics[width=1.01\linewidth]{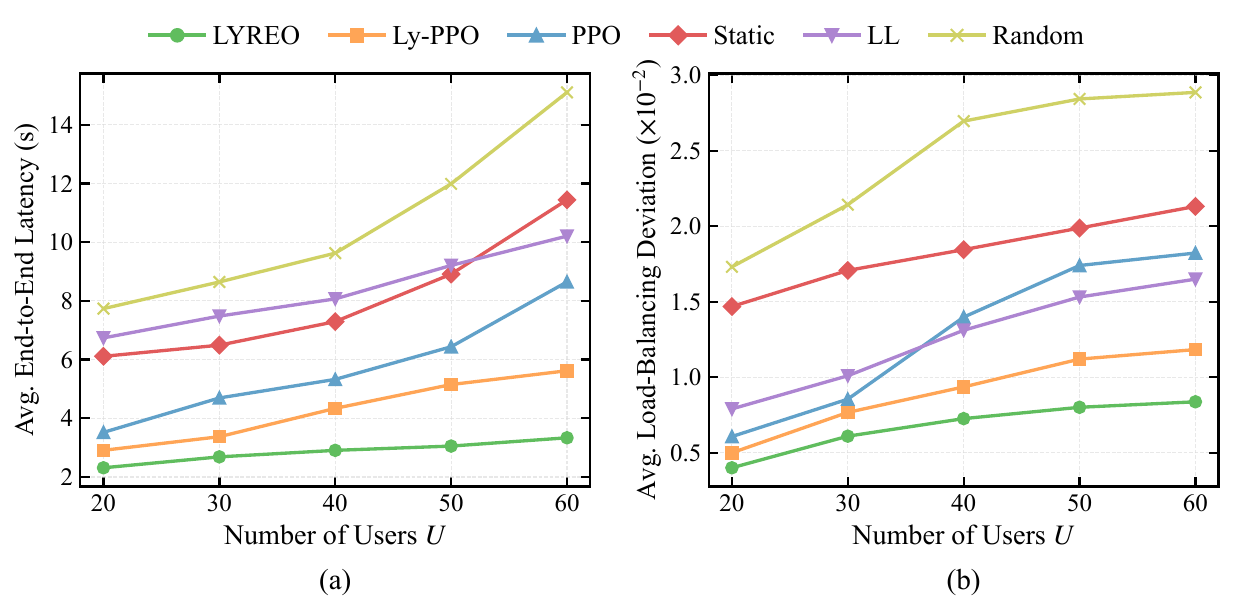}
    \caption{Performance comparison under varying numbers of users.}
	\label{fig:Q}
    \vspace{-0.1in}
\end{figure}
\begin{figure}[!tb]
	\centering
    \includegraphics[width=1.01\linewidth]{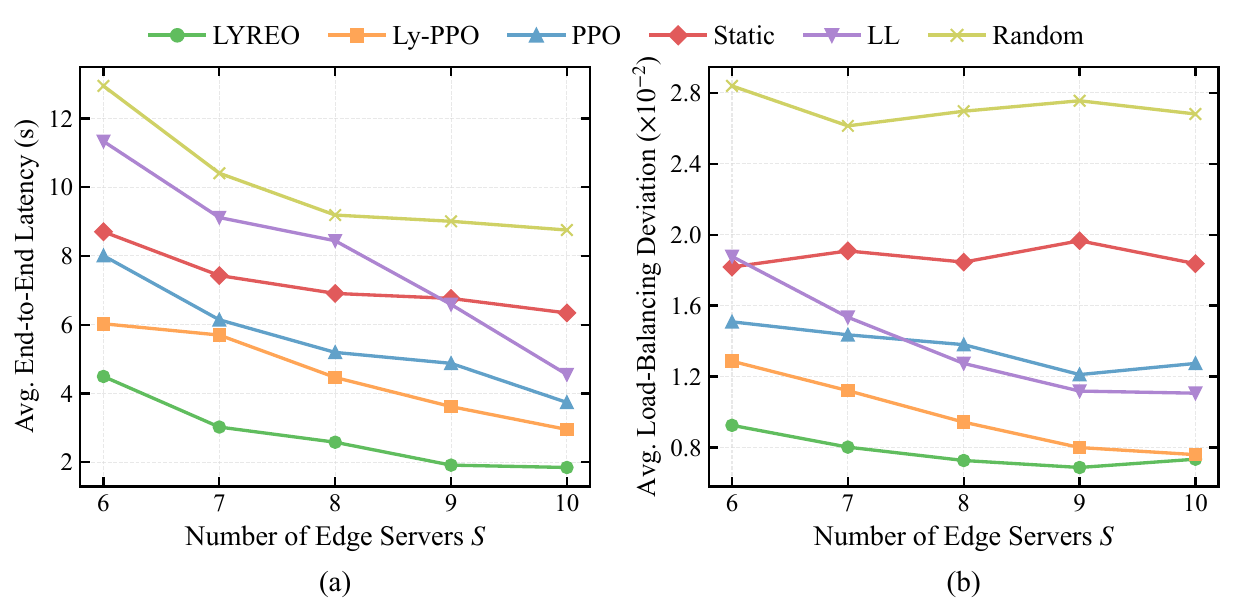}
    \caption{Performance comparison under varying numbers of edge servers.}
	\label{fig:S}
    \vspace{-0.1in}
\end{figure}
Fig.~\ref{fig:S} compares system performance with respect to the number of edge servers, with the number of users fixed at $U=40$.
Increasing $S$ provides more transmission bandwidth, computation capacity, and KV cache capacity, so all algorithms achieve lower latency.
From Fig.~\ref{fig:S}(a), LYREO consistently achieves the best latency performance across every server scale, followed by Ly-PPO and PPO.
Among the non-learning methods, LL improves markedly as the number of servers increases, since least-loaded routing can select from a large pool of lightly occupied servers. Static shows only modest improvement, as its fixed batching strategy limits the benefit of additional servers, while Random remains the worst throughout, owing to its state-agnostic assignment.
The load-balancing deviations in Fig.~\ref{fig:S}(b) show no clear monotonic trend as the number of edge servers increases. 
Specifically, the learning-based approaches generally decrease at first and then flatten or slightly rebound. With a fixed number of requests, more edge servers mitigate workload concentration but also result in sparser per-server demand and a larger action space, making effective scheduling increasingly difficult.
Nevertheless, LYREO maintains the lowest deviation across all settings, showing its stronger ability to exploit additional server choices while preserving system-wide load balance.

\begin{figure*}[!t]
    \centering
    \begin{minipage}[t]{0.31\textwidth}
        \centering
        \includegraphics[width=0.9\linewidth]{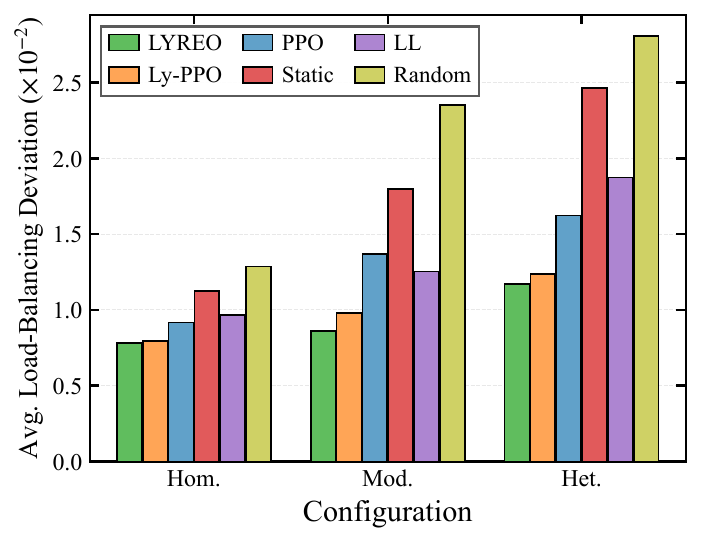}
        \caption{Performance comparison of load-balancing deviation under varying heterogeneity levels.}
        \label{fig:heter}
    \end{minipage}
    \hfill
    \begin{minipage}[t]{0.33\textwidth}
        \centering
        \includegraphics[width=0.92\linewidth]{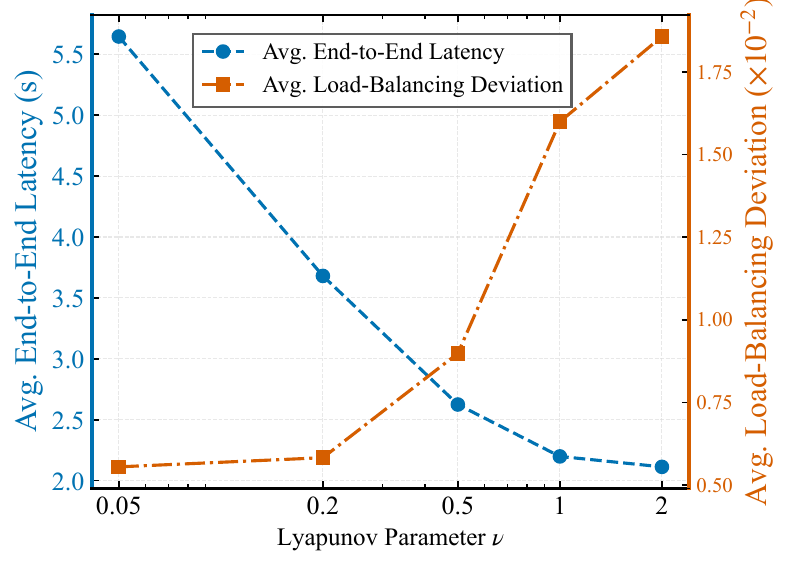}
        \caption{Impact of Lyapunov parameter on system performance.}
        \label{fig:v}
    \end{minipage}
    \hfill
    \begin{minipage}[t]{0.31\textwidth}
        \centering
        \includegraphics[width=0.9\linewidth]{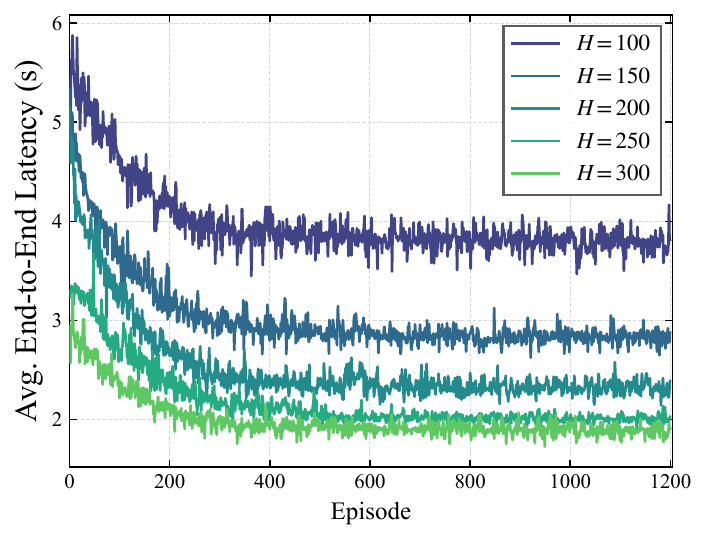}
        \caption{Impact of truncation length on training performance.}
        \label{fig:H}
    \end{minipage}
    \vspace{-0.1in}
\end{figure*}

Fig.~\ref{fig:heter} evaluates the impact of hardware heterogeneity under three distinct configurations:
Homogeneous (Hom.) uses $8\times$ RTX~3090 GPUs; Moderate (Mod.) uses $3\times$ RTX~3080, $3\times$ RTX~3090, and $2\times$ RTX~4090 GPUs; and Heterogeneous (Het.) uses $4\times$ RTX~3080, $1\times$ RTX~3090, and $3\times$ RTX~4090 GPUs.
Average load-balancing deviation increases from Hom. to Mod. and Het. for every method, as growing heterogeneity in computational capacity across edge servers makes balanced workloads harder to maintain.
Under the Hom. configuration, the three learning-based methods achieve the lowest deviations.
As heterogeneity increases, LYREO and Ly-PPO remain the two best-performing methods.
LYREO consistently achieves the lowest deviation and remains slightly better than Ly-PPO, a small yet consistent advantage showing that reward redistribution improves not only latency but also load-balancing performance.
LL performs best among the non-learning methods, a predictable outcome given that its routing is explicitly designed to correct current imbalances, thus providing a natural advantage in the deviation dimension.
Static and Random perform worst under almost all configurations, as fixed batching and state-agnostic routing fail to adapt effectively to server heterogeneity.

\subsection{Hyperparameter Sensitivity}
Fig.~\ref{fig:v} illustrates the latency-balancing trade-off of LYREO controlled by the Lyapunov parameter $\nu$.
Increasing $\nu$ from $0.05$ to $2$ places greater weight on the latency term in the drift-plus-penalty objective. Accordingly, end-to-end latency falls from about $5.6$~s to $2.1$~s, with most of the gain obtained by $\nu=0.5$, while the corresponding load-balancing deviation rises slowly at first and then increases from approximately $0.008$ at $\nu=0.5$ to $0.016$ and $0.019$ at $\nu=1$ and $2$, respectively.
As expected, latency minimization comes at the cost of virtual-queue stabilization.
Notably, $\nu=0.5$ marks a turning point in the trade-off: latency improvements become markedly less responsive to further increases in $\nu$, whereas the load-balancing deviation grows increasingly sensitive.

Finally, Fig.~\ref{fig:H} shows the effect of the return predictor's truncation length on LYREO's average end-to-end latency.
While all configurations converge, a longer $H$ consistently yields a lower final latency.
A larger $H$ exposes the return predictor to more of the actual delayed reward, reducing its reliance on the bootstrapped estimation and yielding more accurate value predictions. This behavior matches our theoretical analysis of LYREO, where the policy performance gap shrinks as $H$ increases.
However, the improvement becomes marginal beyond $H=250$, suggesting that most decision-relevant dependencies in the system can already be captured within a bounded horizon. Meanwhile, a longer truncation length extends the LSTM-based predictor's recurrent trajectories, increasing training complexity.
Consistent with this, $H = 300$ converges to a latency level close to that of $H=250$ but exhibits larger fluctuations, indicating diminishing benefits and increased training instability.



\section{Conclusion}\label{concl}
This paper studied LLM inference request scheduling for agentic AI services, aiming to minimize long-term average end-to-end latency while maintaining load balance across heterogeneous edge servers. 
To capture fine-grained LLM serving processes, we developed a system model that jointly characterizes wireless transmission, multi-stage inference, and KV cache evolution throughout each request's lifecycle. 
Based on this model, we quantified edge server workload using the normalized KV cache memory-time consumption and formulated a long-term load-balancing constraint. 
We proposed the LYREO approach, which converts accumulated load imbalance into per-slot scheduling guidance and redistributes delayed outcomes to their responsible decisions, with sequence truncation and value bootstrapping enabling return estimation in the infinite-horizon setting.
Extensive evaluations demonstrated that LYREO consistently reduces both end-to-end latency and load-balancing deviation, outperforming existing baselines.

\bibliographystyle{IEEEtran}
\bibliography{modified}

@ARTICLE{zhang25twc,
  author={Zhang, Xinyuan and others},
  journal={IEEE Trans. Wireless Commun.},
  title={Beyond the Cloud: Edge Inference for Generative Large Language Models in Wireless Networks},
  year={2025},
  volume={24},
  number={1},
  pages={643-658}
}

@ARTICLE{li23tvt,
  author={Li, Zhen and Yang, Chao and Huang, Xumin and Zeng, WeiLiang and Xie, Shengli},
  journal={IEEE Trans. Veh. Technol.},
  title={CoOR: Collaborative Task Offloading and Service Caching Replacement for Vehicular Edge Computing Networks},
  year={2023},
  volume={72},
  number={7},
  pages={9676-9681}
}

@INPROCEEDINGS{li25lcn,
  author={Li, Tan and Gong, Yanming},
  booktitle={Proc. IEEE 50th Conf. Local Comput. Netw. (LCN)},
  title={Two-Sided Matching for Batch-Aware LLM Request Scheduling in Edge Networks},
  year={2025},
  volume={},
  number={},
  pages={1-7}
}

@ARTICLE{huang25iotj,
  author={Huang, Hualong and others},
  journal={IEEE Internet Things J.},
  title={Dynamic Model Deployment, Batch Scheduling, and Resource Allocation in MLLM-Enabled Edge–Cloud Networks: A Multiagent Two-Timescale DRL Approach},
  year={2025},
  volume={12},
  number={23},
  pages={50818-50835}
}

@inproceedings{kwon23sosp,
  author={Kwon, Woosuk and others},
  title={Efficient Memory Management for Large Language Model Serving with PagedAttention},
  year={2023},
  booktitle={Proc. 29th Symp. Operating Syst. Princ. (SOSP)},
  pages={611–626},
  numpages={16},
  location={Koblenz, Germany}
}

@ARTICLE{zheng26tmc,
  author={Zheng, Tong and others},
  journal={IEEE Trans. Mobile Comput.},
  title={Joint Optimization of Dynamic Batching and Adaptive Partitioning for Distributed LLMs Inference in Mobile Edge Computing},
  year={2026},
  volume={25},
  number={6},
  pages={8747-8763}
}

@article{cheng24arxiv,
  title={Slice-level scheduling for high throughput and load balanced llm serving},
  author={Cheng, Ke and others},
  journal={arXiv preprint arXiv:2406.13511},
  year={2024}
}

@INPROCEEDINGS{li24hpec,
  author={Li, Baolin and Jiang, Yankai and Gadepally, Vijay and Tiwari, Devesh},
  booktitle={Proc. IEEE High Perform. Extreme Comput. Conf. (HPEC)},
  title={LLM Inference Serving: Survey of Recent Advances and Opportunities},
  year={2024},
  volume={},
  number={},
  pages={1-8}
}

@inproceedings{yu22osdi,
  title={Orca: A distributed serving system for Transformer-Based generative models},
  author={Yu, Gyeong-In and Jeong, Joo Seong and Kim, Geon-Woo and Kim, Soojeong and Chun, Byung-Gon},
  booktitle={Proc. 16th USENIX Symp. Oper. Syst. Design Implementation (OSDI)},
  pages={521--538},
  year={2022}
}

@book{neely10morgan,
  title={Stochastic network optimization with application to communication and queueing systems},
  author={Neely, Michael},
  year={2010},
  publisher={Morgan \& Claypool Publishers}
}

@article{arjona19neurips,
  title={Rudder: Return decomposition for delayed rewards},
  author={Arjona-Medina, Jose A and others},
  journal={Proc. Adv. Neural Inf. Process. Syst. (NeurIPS)},
  volume={32},
  year={2019}
}

@article{chen21tcom,
  title={RAN information-assisted TCP congestion control using deep reinforcement learning with reward redistribution},
  author={Chen, Minghao and others},
  journal={IEEE Trans. Commun.},
  volume={70},
  number={1},
  pages={215--230},
  year={2021},
  publisher={IEEE}
}

@ARTICLE{gui26tnse,
  author={Gui, Jinsong and Li, Zhengyang and Zhang, Jingjing and Deng, Xiaoheng and Min, Geyong},
  journal={IEEE Trans. Netw. Sci. Eng.},
  title={Multi-Heterogeneous-Agent DRL for Efficient Congestion Control With Reward Redistribution in Space–Air–Ground Integrated Networks},
  year={2026},
  volume={13},
  number={},
  pages={3035-3052}
}

@article{schulman17arxiv,
  title={Proximal policy optimization algorithms},
  author={Schulman, John and Wolski, Filip and Dhariwal, Prafulla and Radford, Alec and Klimov, Oleg},
  journal={arXiv preprint arXiv:1707.06347},
  year={2017}
}

@inproceedings{zheng24arxiv,
  title={Lmsys-chat-1m: A large-scale real-world llm conversation dataset},
  author={Zheng, Lianmin and others},
  booktitle={Proc. Int. Conf. Learn. Representations (ICLR)},
  volume={2024},
  pages={22225--22257},
  year={2024}
}

@ARTICLE{zhang25iotj,
  author={Zhang, Mingjin and Shen, Xiaoming and Cao, Jiannong and Cui, Zeyang and Jiang, Shan},
  journal={IEEE Internet Things J.},
  title={EdgeShard: Efficient LLM Inference via Collaborative Edge Computing},
  year={2025},
  volume={12},
  number={10},
  pages={13119-13131}
}

@inproceedings{younesi25ucc,
  title={Splitwise: Collaborative edge--cloud inference for LLMs via Lyapunov-assisted DRL},
  author={Younesi, Abolfazl and others},
  booktitle={Proc. IEEE/ACM Int. Conf. Utility Cloud Comput.},
  pages={1--11},
  year={2025}
}

@INPROCEEDINGS{qiu25fcn,
  author={Qiu, Ningyuan and others},
  booktitle={Proc. Int. Conf. Future Commun. Netw. (FCN)},
  title={Joint Request Batching and Worker Assignment in AI-RAN Edge Inference Systems},
  year={2025},
  volume={},
  number={},
  pages={1-6}
}

@ARTICLE{jiang26jsac,
  author={Jiang, Feibo and others},
  journal={IEEE J. Sel. Areas Commun.},
  title={From Large AI Models to Agentic AI: A Tutorial on Future Intelligent Communications},
  year={2026},
  volume={44},
  number={},
  pages={3507-3540}
}

@ARTICLE{zhang26comst,
  author={Zhang, Ruichen and others},
  journal={IEEE Commun. Surv. Tutor.},
  title={Toward Edge General Intelligence With Agentic AI and Agentification: Concepts, Technologies, and Future Directions},
  year={2026},
  volume={28},
  number={},
  pages={4285-4318}
}

@ARTICLE{mou26tsc,
  author={Mou, Fangyi and Tang, Zhiqing and Jia, Weijia and Zhao, Wei},
  journal={IEEE Trans. Serv. Comput.},
  title={Adaptive Request Scheduling and Load Balancing for Edge Deployed Large Language Models},
  year={2026},
  volume={19},
  number={2},
  pages={934-947}
}

@INPROCEEDINGS{ma26infocom,
  author={Ma, Xuehao and Zhou, Huan and Wu, Tong and Fan, Xinggang},
  booktitle={IEEE Conf. Comput. Commun. (IEEE INFOCOM)},
  title={Preference-Aware Task Routing for Edge-Cloud Hierarchical Large Language Model Inference},
  year={2026},
  volume={},
  number={},
  pages={2985-2986}
}

@INPROCEEDINGS{jin26infocom,
  author={Jin, Hengli and Li, Shuo and Gregory, Mark A},
  booktitle={Proc. IEEE Conf. Comput. Commun. (IEEE INFOCOM)},
  title={GenSched: Phase-Aware Generative Scheduling for LLM Inference in Heterogeneous Edge Networks},
  year={2026},
  volume={},
  number={},
  pages={1-6}
}

@INPROCEEDINGS{mek25icc,
  author={Mekrache, Abdelkader and Ksentini, Adlen and Verikoukis, Christos},
  booktitle={Proc. IEEE Int. Conf. Commun. (ICC)},
  title={DRL-Enabled SLO-Aware Task Scheduling for Large Language Models in 6G Networks},
  year={2025},
  volume={},
  number={},
  pages={813-818}
}

@ARTICLE{li25tsc,
  author={Li, Yandi and others},
  journal={IEEE Trans. Serv. Comput.},
  title={Cloud-Edge System for Scheduling Unpredictable LLM Requests With Combinatorial Bandit},
  year={2025},
  volume={18},
  number={6},
  pages={3567-3580}
}

@misc{li25arxiv,
  title={LLM Bandit: Cost-Efficient LLM Generation via Preference-Conditioned Dynamic Routing},
  author={Yang Li},
  year={2025},
  eprint={2502.02743},
  archiveprefix={arXiv},
  primaryclass={cs.LG},
  url={https://arxiv.org/abs/2502.02743}
}

@ARTICLE{zhu26twc,
  author={Zhu, Bingjie and Chen, Zhixiong and Zhao, Liqiang and Shin, Hyundong and Nallanathan, Arumugam},
  journal={IEEE Trans. Wireless Commun.},
  title={Enabling Efficient Large Language Model Inference Over Wireless Networks With Caching},
  year={2026},
  volume={25},
  number={},
  pages={18326-18343}
}

@article{tang26arxiv,
  title={GELATO: Generative Entropy-and Lyapunov-based Adaptive Token Offloading for Device-Edge Speculative LLM Inference},
  author={Tang, Zengzipeng and Sun, Yuxuan and Chen, Wei and Ding, Jianwen and Ai, Bo},
  journal={arXiv preprint arXiv:2605.10124},
  year={2026}
}

@ARTICLE{he24tmc,
  author={He, Ying and Fang, Jingcheng and Yu, F. Richard and Leung, Victor C.},
  journal={IEEE Trans. Mobile Comput.},
  title={Large Language Models (LLMs) Inference Offloading and Resource Allocation in Cloud-Edge Computing: An Active Inference Approach},
  year={2024},
  volume={23},
  number={12},
  pages={11253-11264}
}

@INPROCEEDINGS{li24iccc,
  author={Li, Jinrong and Han, Biao and Li, Sudan and Wang, Xiaoyan and Li, Jie},
  booktitle={Proc. IEEE/CIC Int. Conf. Commun. China (ICCC)},
  title={CoLLM: A Collaborative LLM Inference Framework for Resource-Constrained Devices},
  year={2024},
  volume={},
  number={},
  pages={185-190}
}

@inproceedings{sun24osdi,
  title={Llumnix: Dynamic scheduling for large language model serving},
  author={Sun, Biao and others},
  booktitle={Proc. 18th USENIX Symp. Oper. Syst. Design Implementation (OSDI)},
  pages={173--191},
  year={2024}
}

@article{chen26arxiv,
  title={A universal load balancing principle and its application to large language model serving},
  author={Chen, Zixi and others},
  journal={arXiv preprint arXiv:2601.17855},
  year={2026}
}

@ARTICLE{yi23tmc,
  author={Yi, Changyan and others},
  journal={IEEE Trans. Mobile Comput.},
  title={Workload Re-Allocation for Edge Computing With Server Collaboration: A Cooperative Queueing Game Approach},
  year={2023},
  volume={22},
  number={5},
  pages={3095-3111}
}

@ARTICLE{zhao26tccn2,
  author={Zhao, Changyuan and others},
  journal={IEEE Trans. Cogn. Commun. Netw.},
  title={Edge General Intelligence Through World Models, Large Language Models, and Agentic AI: Fundamentals, Solutions, and Challenges},
  year={2026},
  volume={12},
  number={},
  pages={5649-5675},
}

\end{document}